\def\year{2021}\relax
\documentclass[letterpaper]{article} 
\usepackage{aaai21}  
\usepackage{times}  
\usepackage{helvet} 
\usepackage{courier}  
\usepackage[hyphens]{url}  
\usepackage{graphicx} 
\usepackage{natbib}  
\usepackage{caption} 
\usepackage{amssymb,amsfonts,dsfont,amsthm}
\usepackage{amsmath}
\usepackage[mathscr]{eucal}
\DeclareMathOperator*{\argmax}{arg\,max}
\DeclareMathOperator*{\argmin}{arg\,min}
\DeclareMathOperator{\EX}{\mathbb{E}}%
\usepackage{xspace}
\usepackage{xstring}
\newcommand{\citeN}[1]{\citeauthor{#1}~(\citeyear{#1})}

\usepackage{pgfplots}
\usepackage{tikz}
\usetikzlibrary{shapes.geometric}
\usetikzlibrary{positioning,shapes,shadows,arrows,calc}
\usetikzlibrary{intersections}
\usetikzlibrary{arrows.meta,calc,decorations.markings,math,arrows.meta}
\pgfplotsset{compat=1.15}
\pgfdeclarelayer{background}
\pgfdeclarelayer{foreground}
\pgfsetlayers{background,main,foreground}

\tikzstyle{activity}=[rectangle, draw=black, rounded corners, text centered, text width=7em, fill=white, drop shadow]
\tikzstyle{data}=[rectangle, draw=black, text centered, fill=black!10, text width=6em, drop shadow]
\tikzstyle{myarrow}=[->, thick]

\usepackage{todonotes}
\presetkeys{todonotes}{inline}{}
\usepackage{subfig}
\usepackage{multirow}
\usepackage{booktabs}
\usepackage{siunitx}
\usepackage{tabularx}
\usepackage{color, colortbl}
\definecolor{Gray}{gray}{0.9}

\newcommand{\techreport}[0]{\citeN{speck-et-al-arxiv2020}}

\let\oldcite=\cite
\renewcommand\cite[1]{\ifthenelse{\equal{#1}{NEEDED}}{[{\color{blue}citation~needed}]}{\oldcite{#1}}}

\newcommand{\ctrlpol}[0]{\tilde{\pi}}
\newcommand{\ctrlpolspace}[0]{\tilde{\Pi}}
\newcommand{\insts}[0]{\ensuremath \mathcal{I}}
\newcommand{\instance}[0]{\mathit{i}}  

\sfcode`\.=1001
\sfcode`\?=1001
\sfcode`\!=1001
\sfcode`\:=1001
\newcommand{\autocase}[1]{\ifnum\ifhmode\spacefactor\else2000\fi>1000 \uppercase{#1}\else#1\fi}
\newcommand{\agent}[1][]{\autocase{c}ontroller#1\xspace} 

\newcommand{\target}[1][]{\autocase{a}lgorithm#1\xspace}

\newcommand{\task}[1][]{\autocase{i}nstance#1\xspace}

\newcommand{\ctrl}[0]{\autocase{d}ynamic algorithm configuration\xspace}
\newcommand{\CTRL}[0]{DAC}

\newcommand{\ptask}[0]{\ensuremath i}
\newcommand{\fulltask}{\ensuremath {\ptask = \langle \vars, \init, \operators, \goal \rangle}\xspace}
\newcommand{\vars}{\ensuremath {\mathcal V}\xspace}
\newcommand{\domain}{\ensuremath {D}\xspace}
\newcommand{\operators}{\ensuremath {\mathcal O}\xspace}
\newcommand{\init}{\ensuremath {s_0}\xspace}
\newcommand{\goal}{\ensuremath {s_\star}\xspace}
\newcommand{\pre}{\ensuremath {\mathit{pre}}\xspace}
\newcommand{\eff}{\ensuremath {\mathit{eff}}\xspace}

\newcommand{\states}{\ensuremath {\mathcal S}\xspace}
\newcommand{\plan}[0]{\ensuremath \pi}
\newcommand{\heus}[0]{\ensuremath H}
\newcommand{\gbf}{greedy best-first search}

\theoremstyle{plain}
\newtheorem{theorem}{Theorem}
\newtheorem{proposition}[theorem]{Proposition}

\newtheorem{corollary}[theorem]{Corollary}
\theoremstyle{definition}

\theoremstyle{remark}

\newcommand{\name}[1]{\textsc{#1}}
\newcommand*{\QED}{\hfill\ensuremath{\square}}

\usepackage{paralist}

\title{Learning Heuristic Selection with Dynamic Algorithm
Configuration}
\author {
    David Speck\textsuperscript{\rm 1,$*$},
    Andr{\'e} Biedenkapp\textsuperscript{\rm 1,$*$},
    Frank Hutter\textsuperscript{\rm 1,\rm 2},
    Robert Mattm{\"u}ller\textsuperscript{\rm 1},
    Marius Lindauer\textsuperscript{\rm 3}\\
}

\affiliations {
    \textsuperscript{$*$}Contact Author, Equal contribution \\
    \textsuperscript{\rm 1}University  of  Freiburg, Freiburg, Germany \\
    \textsuperscript{\rm 2}Bosch Center for Artificial Intelligence, Renningen, Germany \\
    \textsuperscript{\rm 3}Leibniz University Hannover, Hannover, Germany \\
    $\langle$speckd, biedenka, fh, mattmuel$\rangle$@cs.uni-freiburg.de, lindauer@tnt.uni-hannover.de
}

\begin{document}

\maketitle

\begin{abstract}
A key challenge in satisficing planning is to use multiple heuristics within one heuristic search. 
An aggregation of multiple heuristic estimates, for example by taking the maximum, has the disadvantage that bad estimates of a single heuristic can negatively affect the whole search.
Since the performance of a heuristic varies from instance to instance, approaches such as algorithm selection can be successfully applied. 
In addition, alternating between multiple heuristics during the search makes it possible to use all heuristics equally and improve performance.
However, all these approaches ignore the internal search dynamics of a planning system, which can help to select the most useful heuristics for the current expansion step.
We show that dynamic algorithm configuration can be used for dynamic heuristic selection which takes into account the internal search dynamics of a planning system. 
Furthermore, we prove that this approach generalizes over existing approaches and that it can exponentially improve the performance of the heuristic search. 
To learn dynamic heuristic selection, we propose an approach based on reinforcement learning and show empirically that domain-wise learned policies, which take the internal search dynamics of a planning system into account, can exceed existing approaches.
\end{abstract}

\section{Introduction}
Heuristic forward search is one of the most popular and successful techniques in classical planning. 
Although there is a large number of heuristics, it is known that the performance, i.e., the informativeness, of a heuristic varies from instance to instance \cite{wolpert-macready-tr1995}. 
While in optimal planning it is easy to combine multiple admissible heuristic estimates using the maximum, in satisficing planning the estimates of inadmissible heuristics are difficult to combine in general \cite{roeger-helmert-icaps2010}. 
The reason for this is that highly inaccurate and uninformative estimates of a heuristic can have a negative effect on the entire search process when aggregating all estimates. 
Therefore, an important task in satisficing planning is to utilize multiple heuristics within one heuristic search. 

\citeN{roeger-helmert-icaps2010} showed the promise of searching with multiple heuristics, maintaining a set of heuristics, each associated with a separate open list to allow switching between such heuristics. 
This bypasses the problem of aggregating different heuristic estimates, while the proposed alternating procedure uses each heuristic to the same extent. 
Another direction is the selection of the best algorithm a priori based on the characteristics of the present planning instance 
\cite{cenamor-et-al-jair2016,sievers-et-al-aaai2019}.
In other words, different search algorithms and heuristics are part of a portfolio from which one is selected to solve a particular problem instance. 
This automated process is referred to as algorithm selection~\cite{rice-aic1976} while optimization of algorithm parameters is referred to as algorithm configuration \cite{hutter-et-al-jair2009}. 
Both methodologies have been successfully applied to planning  \cite{fawcett-et-al-ipc2011a,fawcett-et-al-icaps2014,seipp-et-al-aaai2015,sievers-et-al-aaai2019} and various other areas of artificial intelligence, such as machine learning \cite{snoek-et-al-nips2012} or satisfiability solving \cite{hutter-et-al-aij2017}.
However, algorithm selection and configuration ignore the non-stationarity of which configuration performs well. In order to remedy this,  
\citeN{biedenkapp-et-al-ecai2020} showed that the problem of selecting and adjusting configurations during the search based on the current solver state and search dynamics can be modelled as a contextual Markov decision process and addressed by standard reinforcement learning methods.

In planning, there is little work that takes into account the search dynamics of a planner to decide which planner to use. \citeN{cook-huber-smc2016} showed that switching between different heuristic searches (planners) based on the search dynamics obtained during a search leads to better performance than a static selection of a heuristic. 
However, in this approach, several disjoint searches (planners) are executed, which do not share the search progress  \cite{aine-likhackev-icaps2016}. 
\citeN{ma-et-al-aaai2020} showed that a portfolio-based approach that can switch the planner at halftime, depending on the performance of the previously selected one, can improve performance over a simple algorithm selection at the beginning.
Recent works have investigated switching between different search strategies depending on the internal search dynamics of a planner \cite{gomoluch-etal-icaps2019,gomoluch-etal-icaps2020}.
One approach that shares the search progress is to maintain multiple heuristics as separate open lists \cite{roeger-helmert-icaps2010}. 
Furthermore, it has been shown that boosting, i.e., giving preference to heuristics that have recently made progress, can improve search performance \cite{richter-helmert-icaps2009}. 
While in these works heuristic values are computed for each state, \citeN{domshlak-et-al-aaai2010} investigated the question, whether the time spent for the computation of the heuristic value for a certain state pays off.

Another avenue of work considers how to ``directly'' create or learn new heuristic functions.
One example is the work of \citeN{ferber-et-al-ecai2020}, which utilizes supervised learning to learn a heuristic function where the input is the planning (world) state itself.
\citeN{thayer-et-al-icaps2011} showed that admissible heuristics can be transformed online, into inadmissible heuristics, which makes it possible to tailor a heuristic to a specific planning instance.

In this work, we introduce and define dynamic algorithm configuration \cite{biedenkapp-et-al-ecai2020} for planning, by learning a policy that dynamically selects a heuristic within a search with multiple open lists \cite{roeger-helmert-icaps2010} based on the current search dynamics.
We prove that a dynamic adjustment of heuristic selection during the search can exponentially improve the search performance of a heuristic search compared to a static heuristic selection or a \emph{non-adaptive} policy like alternating.
Furthermore, we show that such a dynamic control policy is a strict generalization of other already existing approaches to heuristic selection.
We also propose a set of state features describing the current search dynamics and a reward function for training a reinforcement learning agent.
Finally, an empirical evaluation shows that it is possible to learn a dynamic control policy on a \textit{per-domain basis} that outperforms approaches that do not involve search dynamics, such as ordinary heuristic search with a single heuristic and alternating between heuristics.

\section{Background}
We first introduce classical planning, then discuss greedy best-first search with multiple heuristics, and finally present the concept of dynamic algorithm configuration based on reinforcement learning. Note that the terminology and notation of planning and reinforcement learning are similar, so we use the symbol $\sim$ for all notations directly related to reinforcement learning; e.g. $\plan$ denotes a plan of a planning task, while $\tilde{\pi}$ is a policy obtained by reinforcement learning. 

\subsection{Classical Planning}
A problem instance or task in classical planning, modeled in the \name{sas}$^+$ formalism \cite{backstrom-nebel-compint1995}, is a tuple \fulltask{} consisting of four components. \vars{} is a finite set of state variables, each associated with a finite domain $\domain_v$. A fact is a pair $(v, d)$, where $v \in \vars$ and $d \in \domain_v$, and a partial variable assignment over $\vars$ is a consistent set of facts, i.e., a set that does not contain two facts for the same variable. If $s$ assigns a value to each $v \in \vars$, $s$ is called a state. States and partial variable assignments are functions which map variables to values, i.e., $s(v)$ is the value
of variable $v$ in state $s$ (analogous for partial variable assignments).  
\operators{} is a set of operators, where an operator is a pair $o = \langle \pre_o, \eff_o \rangle$ of partial variable assignments called preconditions and effects, respectively. Each operator has cost $c_o \in \mathbb{N}_0$.
The state \init{} is called the initial state and the partial variable assignment \goal{} specifies the goal condition, which defines all possible goal states $S_{\star}$. With \states{} we refer to the set of all states defined over \vars{}, and with $|\ptask{}|$ we refer to the size of the planning task $\ptask{}$, i.e., the number of operators and facts.

We call an operator $o \in \operators{}$ applicable in state $s$ iff $\pre_o$ is satisfied in $s$, i.e., $s \models \pre_o$. Applying operator $o$ in state $s$ results in a state $s'$ where $s'(v) = \eff_o(v)$ for all variables $v \in \vars$ for which $\eff_o$ is defined and $s'(v) = s(v)$ for all other variables. We also write $s[o]$ for $s'$. The objective of classical planning is to determine a plan, which is defined as follows.
	A \emph{plan} $\pi = \langle o_0, \dots, o_{n-1} \rangle$ for planning task
	$\ptask$ is a sequence of applicable operators which generates a sequence of
	states $s_0, \dots, s_n$, where $s_n \in S_\star$ is a goal state and $s_{i+1} = 
	s_i[o_i]$ for all $i = 0, \dots, n-1$. The cost of plan $\pi$ is the sum of its operator costs.

Given a planning task, the search for a good plan is called satisficing planning. In practice, heuristic search algorithms such as greedy best-first search 
have proven to be one of the dominant search strategies for satisficing planning. 

\subsection{Greedy Search with Multiple Heuristics}
Greedy best-first search is a pure heuristic search which tries to estimate the distance to a goal state by means of a heuristic function. 
A heuristic is a function $h: \states \mapsto \mathbb{N}_{0} \cup \{\infty\}$, which estimates the cost to reach a goal state from a state $s \in \states$. 
The perfect heuristic $h^\star$ maps each state $s$ to the cost of the cheapest path from $s$ to any goal state $s_\star \in S_\star$.
The idea of greedy best-first search with a single heuristic $h$ is to start with the initial state and to expand the most promising states based on $h$ until a goal state is found \cite{pearl-1984}. 
During the search, relevant states are stored in an open list that is sorted by the heuristic values of the contained states in ascending order so that the state with the lowest heuristic values, i.e., the most promising state, is at the top.
More precisely, in each step a state $s$ with minimal heuristic value is expanded, i.e., its successors $S' = \{s[o] ~|~ o \in O, s \models \pre_o\}$ are generated and states $s' \in S'$ not already expanded are added to the open list according to their heuristic values $h(s')$. Within an open list, for states with the same heuristic value ($h$-value) the tie-breaking rule that is used is according to the first-in-first-out principle.

In satisficing planning it is possible to combine multiple heuristic values for the same state in arbitrary ways. 
It has been shown, however, that the combination of several heuristic values into one, e.g. by taking the maximum or a (weighted) sum, does not lead to informative heuristic estimates \cite{roeger-helmert-icaps2010}. 
This can be explained by the fact that if one or more heuristics provide very inaccurate values, the whole expansion process is affected.
\citeN{helmert-jair2006} introduced the idea
to maintain multiple heuristics $\heus = \{h_0,\dots,h_{n-1}\}$ within one greedy best-first search. 
More precisely, it is possible to maintain a separate open list for each heuristic $h \in \heus$ and switch between them at each expansion step while always expanding the most promising state of the currently selected open list. 
The generated successor states are then evaluated with \emph{each} heuristic and added to the corresponding open lists. 
This makes it possible to share the search progress \cite{aine-likhackev-icaps2016}. 
Especially, an alternation policy, in which all heuristics are selected one after the other in a cycle such that all heuristics are treated and used equally, has proven to be an efficient method~\cite{roeger-helmert-icaps2010}.
Such equal use of heuristics can help to progress the search space towards a goal state, even if only one heuristic is informative.
However, in some cases it is possible to infer that some heuristics are currently, i.e., in the current region of the search space, more informative than others, which is ignored by a strategy like alternation. 
More precisely, with alternation, the choice of the heuristic depends only on the current time step and not on the current search dynamics or planner state.
In general, it is possible to dynamically select a heuristic based on internal information provided by the planner.
This is the key idea behind our approach described in the following.

\subsection{Dynamic Algorithm Configuration}
Automated algorithm configuration (AC) has proven a powerful approach to leveraging the full potential of algorithms.
Standard AC views the algorithms being optimized as black boxes, thereby ignoring an algorithm's temporal behaviour
and ignoring that an optimal configuration might be non-stationary~\cite{arfaee-et-al-aij2011}.
\ctrl (\CTRL) is a new meta-algorithmic framework that makes it possible to learn to adjust the parameters of an algorithm given a description of the algorithm's behaviour~\cite{biedenkapp-et-al-ecai2020}.

We first describe \CTRL{} on a high level.
Given a parameterized algorithm $A$ with its configuration space $\tilde{\Theta}$, a set of problem instances $\insts$ the algorithm has to solve, 
a state description $\tilde{s}^\instance_t$ of the algorithm $A$ solving an instance $\instance \in \insts$ at step $t \in \mathbb{N}_0$, and a reward signal $\tilde{r}$ assessing the reward (e.g., runtime or number of state expansions) of using a control policy $\ctrlpol \in \ctrlpolspace$ to control $A$ on an instance $i \in \insts$, the goal is to find a (\emph{dynamic}) control policy
$\ctrlpol^* : \mathbb{N}_0 \times \tilde{\mathcal{S}} \times \insts \to \tilde{\Theta}$. This policy  adaptively chooses a configuration $\tilde{\theta} \in \tilde{\Theta}$ given a state $\tilde{s}_t \in \tilde{\mathcal{S}}$ of $A$ at time $t \in \mathbb{N}_0$ to optimize the reward of $A$ across the set of instances $\mathcal{S}$, i.e., 
$\ctrlpol^* \in \argmax_{\ctrlpol \in \ctrlpolspace} \EX[\tilde{r}(\ctrlpol,i)]$.
Note that the current time step $t \in \mathbb{N}_0$ and instance $i \in \insts$ can be encoded in the state description $\tilde{\mathcal{S}}$ of an algorithm $A$, which leads to a dynamic control policy, defined as $\ctrlpol_{\text{dac}} : \tilde{\mathcal{S}} \to \tilde{\Theta}$.

Figure~\ref{fig:pc} depicts the interaction between a control policy $\ctrlpol$ and a planning system $A$ schematically. At each time step $t$, the planner sends the current internal state $\tilde{s}^\instance_t$ and the corresponding reward $\tilde{r}^\instance_t$ to the control policy $\ctrlpol$ based on which the controller decides which parameter setting $h_{t+1}\in\tilde{\Theta}$ to use. The planner progresses according to the decision to the next internal state $\tilde{s}^i_{t+1}$ with reward $\tilde{r}^i_{t+1}$. 
This formalisation of \ctrl makes it possible to recover prior meta-algorithmic frameworks as special cases which we discuss below. 

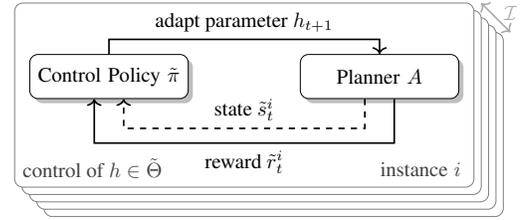
\begin{figure}[tbp]
\centering
\resizebox{0.9\columnwidth}{!}{
    \begin{tikzpicture}[node distance=2.1cm]

\node (Agent) [activity,
minimum height=.75cm] {Control Policy $\ctrlpol$};

\node (Algo) [activity, right of=Agent, xshift=2.5cm,
minimum height=.75cm] {Planner $A$};

\begin{pgfonlayer}{background}
\path (Agent -| Agent.west)+(-0.25,1.25) node (resUL) {};
\path (Algo.east |- Algo.south)+(0.25,-1.5) node (resBR) {};

\path [rounded corners, draw=black!50, fill=white] ($(resUL)+(0.5, -0.5)$) rectangle ($(resBR)+(0.5, -0.5)$);
\path [rounded corners, draw=black!50, fill=white] ($(resUL)+(0.375, -0.375)$) rectangle ($(resBR)+(0.375, -0.375)$);
\path [rounded corners, draw=black!50, fill=white] ($(resUL)+(0.25, -0.25)$) rectangle ($(resBR)+(0.25, -0.25)$);
\path [rounded corners, draw=black!50, fill=white] ($(resUL)+(0.125, -0.125)$) rectangle ($(resBR)+(0.125, -0.125)$);

\path [rounded corners, draw=black!50, fill=white] (resUL) rectangle (resBR);
\path (resBR)+(-.9,0.3) node [text=black!75] {\task $\instance$};
\path (resUL.east |- resBR.north)+(+1.2,0.2) node [text=black!75] {control of $h \in \tilde{\Theta}$};

\end{pgfonlayer}


\draw[myarrow] (Agent.north) -- ($(Agent.north)+(0.0,+0.25)$) -- ($(Algo.north)+(0.0,+0.25)$) node [above,pos=0.5] { adapt parameter $h_{t+1}$} node [below,pos=0.5] {} -- (Algo.north);
\draw[myarrow, dashed] ($(Algo.south)+(-0.25, 0)$) -- ($(Algo.south)+(-0.25, -0.5)$) -- ($(Agent.south)+(0.25, -0.5)$) node [above,pos=0.5] {state $\tilde{s}^i_{t}$} -- ($(Agent.south)+(0.25, 0)$);
\draw[myarrow] ($(Algo.south)+(0.25, 0)$) -- ($(Algo.south)+(0.25, -0.75)$) -- ($(Agent.south)+(-0.25, -0.75)$) node [below,pos=0.5] {reward $\tilde{r}^i_{t}$} -- ($(Agent.south)+(-0.25, 0)$);

\draw[<->, thick, draw=black!32.5] (resBR.east |- resUL.center) -- ($(resBR.east |- resUL.center)+(0.5, -0.5)$);
\path (resBR.east |- resUL.center)+(.5, -0.175) node [text=black!32.5] {$\insts$};
\path (resUL.west |- resUL.center)+(-.5, -0.175) node [text=black!0] {$\insts$}; 

\end{tikzpicture}
}
\caption{Dynamic configuration of parameter $h \in \tilde{\Theta}$ of \target $A$ on an \task
$\instance\in\insts$, at time step $t\in \mathbb{N}_0$.
Until $\instance$ is solved or a budget is exhausted, the \agent adapts parameter $h$, based on the internal state $\tilde{s}^i_t$ of $A$.}
\label{fig:pc}%
\end{figure}%

\section{Dynamic Heuristic Selection}
In this section, we will explain how dynamic algorithm configuration can be used in the context of dynamic heuristic selection and how it differs from time-adaptive or in short \emph{adaptive}
algorithm configuration and algorithm selection, which have already been used in the context of search with multiple heuristics.
\citeN{helmert-jair2006} introduced the idea of maintaining a set of heuristics $\heus$ each associated with a separate open list in order to allow the alternation between such heuristics. 
Considering $\heus$ as the configuration space $\tilde{\Theta}$ of a heuristic search algorithm $A$ and each state expansion as a time step $t$, it is possible to classify different dynamic heuristic selection strategies within the framework of dynamic algorithm configuration.
For example, alternation is an time-adaptive control policy because it maps each time step to a specific heuristic, i.e.,  configuration, independent of the instance or the state of the planner. 
The selection of a particular heuristic depending on the current instance before solving the instance, known as ``portfolio planner'', is an algorithm selection policy that depends only on the instance and not on the current time step or the internal state of the planner. 
Exceptions are policies that compare the heuristic values of states, such as the expansion of the state with the overall minimal heuristic value or according to a Pareto-optimality analysis \cite{roeger-helmert-icaps2010}. 
Such policies depend on the current state of the planner, but ignore the time step and the current instance being solved. 
This indicates that all three components --- instance, time step, and state of the planner --- can be important and helpful in selecting the heuristic for the next state expansion. 
The following summarizes existing approaches to heuristic selection within the framework of algorithm configuration.

\medskip

\begin{compactitem}
    \item \textit{Algorithm Selection}: 
    \begin{itemize}
        \item Policy: $\ctrlpol_{\text{as}}: \mathcal{I} \to \heus$
        \item Example: Portfolios \cite{seipp-et-al-icaps2012,cenamor-et-al-jair2016,sievers-et-al-aaai2019}
    \end{itemize}
    \item \textit{Adaptive Algorithm Configuration}: 
    \begin{itemize}
        \item Policy: $\ctrlpol_{\text{aac}} : \mathbb{N}_0 \to \heus$ 
        \item Example: Alternation \cite{roeger-helmert-icaps2010,seipp-et-al-aaai2015} 
    \end{itemize}
    \item \textit{Dynamic Algorithm Configuration}: 
    \begin{itemize}
        \item Policy: $\ctrlpol_{\text{dac}} : \mathbb{N}_0 \times \tilde{\mathcal{S}} \times \insts \to \heus$
        \item Example: Approach proposed in this paper
    \end{itemize}
\end{compactitem}

\subsection{An Approach based on Reinforcement Learning}

In this section, we describe all the parts required to dynamically configure a planning system so that for each individual time step, a dynamic control policy can decide which heuristic to use based on a dynamic control policy. Here, a time step is a single expansion step of the planning system.

\paragraph{State description.} Learning dynamic configuration policies requires descriptive state features that inform the policy about the characteristics and the behavior of the planning system in the search space.
Preferably, such features are domain-independent, such that the same features can be used for a wide variety of domains. 
In addition, such state features should be cheap to compute in order to keep the overhead as low as possible.

As consequence of both desiderata and the intended learning task we propose to use the following state features computed over the entries contained in the corresponding open list of each heuristic:
\begin{description}
    \item[$\max_h$:] maximum $h$ value for each heuristic $h \in \heus$;
    \item[$\min_h$:]  minimum $h$ value for each heuristic $h \in \heus$;
    \item[$\mu_h$:]  average $h$ value for each heuristic $h \in \heus$;
    \item[$\sigma^2_h$:] variance of the $h$ values for each heuristic $h \in \heus$;
    \item[$\#_h$:] number of entries for each heuristic $h \in \heus$;
    \item[$t$:] current time/expansion step $t \in \mathbb{N}_0$.
\end{description}
To measure progress, we do not directly use the values of each state feature, but compute the \emph{difference of each state feature} between successive time steps $t-1$ and $t$.
The configuration space is a finite set of $n$ heuristics to choose from, i.e., $\tilde{\Theta} = \heus = \{h_0,\dots,h_{n-1}\}$.

The described set of features is a starting point and domain independent, but does not contain any specific context information yet. In general, it is possible to describe an instance or domain with features that describe, for example, the variables, operators or the causal graph \cite{sievers-et-al-aaai2019}.
If the goal is to learn robust policies that can handle highly heterogeneous sets of instances, it is possible to add contextual information about the planning instance at hand, such as the problem size or the required preprocessing steps \cite{fawcett-et-al-icaps2014}, to the state des\-cription. However, in this work, we limit ourselves to domain-wise dynamic control policies and show that the concept of DAC can improve heuristic search for this setting in theory and practice.

\paragraph{Reward function.} Similar to the state description, the reward function we want to optimize should ideally be domain-independent, cheap and quick to compute. 
Since the goal is usually to quickly solve as many tasks as possible, a good reward function should reflect this desire.

We use a simple reward of $-1$ for each expansion step that the planning system has to perform in order to find a solution.
Using this reward function, a configuration policy learns to select heuristics that minimize the expected number of state expansions until a solution is found.
This reward function ignores aspects such as the quality of a plan, but its purpose is to reduce the search effort and thus improve search performance. 
Clearly, it is possible to define other reward functions with, e.g., dense rewards to make learning easier. 
We nevertheless demonstrate that already with our reward function and state features it is possible to learn dynamic control policies, which dominate algorithm selection and adaptive control policies in theory and practice.

\subsection{Dynamic Algorithm Configuration in Theory}
In this section, we investigate the theoretical properties of using \name{dac} for heuristic search algorithms.
In optimal planning, where the goal is to find a plan with minimal cost, the performance of heuristic search can be measured by the number of state expansions \cite{helmert-roeger-aaai2008}. 
This is different for satisficing planning, because plans with different costs can be found and there are generally no ``must expand'' states that need to be expanded to prove that a solution is optimal. 
However, the number of state expansions until \emph{any} goal state is found can be used to measure the guidance of a heuristic or heuristic selection \cite{richter-helmert-icaps2009,roeger-helmert-icaps2010}.

We want to answer the question of whether it can theoretically be beneficial to use dynamic control policies $\ctrlpol_{\text{dac}}$ over algorithm selection policies $\ctrlpol_{\text{as}}$ or adaptive algorithm configuration policies $\ctrlpol_{aac}$. 
Proposition \ref{thm:same_expansion_number} proves that for each heuristic search algorithm in combination with each collection of heuristics there is a dynamic control policy $\ctrlpol_{\text{dac}}$ which is as good as $\ctrlpol_{\text{as}}$ or $\ctrlpol_{aac}$ in terms of state expansions. 

\begin{proposition}\label{thm:same_expansion_number}
Independent of the heuristic search algorithm and the collection of heuristics, for each algorithm selection policy $\ctrlpol_{\text{as}}$ and adaptive algorithm configuration policy $\ctrlpol_{\text{aac}}$ there is a dynamic control policy $\ctrlpol_{\text{dac}}$ which expands at most as many states as $\ctrlpol_{\text{as}}$ and $\ctrlpol_{\text{aac}}$ until a plan is found for a given planning instance.
\end{proposition}
\begin{proof}
DAC policies generalize algorithm selection and adaptive algorithm configuration policies, thus it is always possible to define $\ctrlpol_{\text{dac}}$ as $\ctrlpol_{\text{dac}} = \ctrlpol_{\text{as}}$ or \mbox{$\ctrlpol_{\text{dac}} = \ctrlpol_{\text{aac}}$}.
\end{proof}

With Proposition \ref{thm:same_expansion_number} it follows directly that an optimal algorithm configuration policy $\ctrlpol_{\text{dac}}^*$ is at least as good as an optimal algorithm selection policy $\ctrlpol_{\text{as}}^*$ and an optimal adaptive algorithm configuration policy $\ctrlpol_{\text{aac}}^*$:

\begin{corollary}\label{col:optimal_policy}
Independent of the heuristic search algorithm and the collection of heuristics, an optimal dynamic control policy $\ctrlpol_{\text{dac}}^*$ expands at most as many states as an optimal algorithm selection policy $\ctrlpol_{\text{as}}^*$ and an optimal adaptive algorithm configuration policy $\ctrlpol_{\text{aac}}^*$ until a plan $\plan$ is found for a planning task. \QED
\end{corollary}

It is natural to ask to what extent the use of a dynamic control policy instead of an algorithm selection or an adaptive control policy can improve the search performance of heuristic search. 
We will show that for each algorithm selection policy $\ctrlpol_{\text{as}}$ and adaptive algorithm configuration policy $\ctrlpol_{\text{aac}}$, we can construct a family of planning tasks so that a dynamic control policy $\ctrlpol_{\text{dac}}$ will expand exponentially fewer states until a plan is found. 
For this purpose, we introduce a family of planning instances $\ptask{}_n$ with $O(n)$ propositional variables and $O(n)$ operators.
The induced transition system of $\ptask{}_n$ is visualized in Figure \ref{fig:proof_task}. 
There is exactly one goal path $s_0, s_1, s_2$, which is induced by the unique plan $\plan = \langle o_1, o_2 \rangle$. 
Furthermore, exactly two states are directly reachable from the initial state, $s_1$ and $s_3$. While state $s_1$ leads to the unique goal state $s_2$, from $s_3$ onward exponentially many states $s_4,\dots,s_{2^n-1}$ in $n=|\ptask_n|$, i.e., $\Omega(2^n) = \Omega(2^{|\ptask{}_n|})$, can be reached by the subsequent application of multiple actions. 

\begin{figure}
    \centering
    \resizebox{0.90\columnwidth}{!}{
        \begin{tikzpicture}
    \begin{scope}
        \node[align=center,draw] (s0) at (0,0) {$s_0$\\$h_0(s_0)=5$\\$h_1(s_0)=6$};
        \node[align=center,draw] (s1) at (3,0.85) {$s_1$\\$h_0(s_1)=5$\\$h_1(s_1)=3$};
        \node[align=center,draw,double] (s2) at (6,0.85) {$s_2 \models \goal$\\$h_0(s_2)=0$\\$h_1(s_2)=0$};
        \node[align=center,draw] (s3) at (3,-0.85) {$s_3$\\$h_0(s_3)=3$\\$h_1(s_3)=4$};
        \node[align=center,draw,dashed] (sn) at (6,-0.85) {$s_k \not\models \goal$\\$h_0(s_k)=1$\\$h_1(s_k)=1$\\ $\forall k \in \{4,\dots,2^{n-1}\}$};
        
        \draw[-{Latex[scale=1.2]}] (s0) to node [above, align=center] {$o_1$} (s1);
        \draw[-{Latex[scale=1.2]}] (s0) to node [above, align=center] {$o_3$} (s3);
        \draw[-{Latex[scale=1.2]}] (s1) to node [above, align=center] {$o_2$} (s2);
        \draw[-{Latex[scale=1.2]}] ($(s3.east)+(0,0.5)$) to node [above, align=center] {} ($(sn.west)+(0,0.5)$);
        \draw[-{Latex[scale=1.2]}] (s3) to node [above, align=center] {} (sn);
        \draw[-{Latex[scale=1.2]}] ($(s3.east)+(0,-0.5)$) to node [above, align=center] {} ($(sn.west)+(0,-0.5)$);
    \end{scope}
\end{tikzpicture}
    }
    \caption{Visualization of the induced transition system of the planning task family $\ptask{}_n$.\label{fig:proof_task}}
\end{figure}
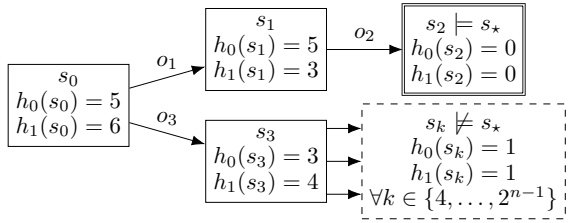

\begin{theorem}\label{thm:exp_expansion_number_aac}
For each adaptive algorithm configuration policy $\ctrlpol_{\text{aac}}$ there exists a family of planning instances $\ptask_n$, a collection of heuristics $\heus$ and a dynamic control policy $\ctrlpol_{\text{dac}}$, so that \gbf{} with $\heus$ and $\ctrlpol_{\text{aac}}$ expands exponentially more states in $|\ptask_n|$ than \gbf{} with $\heus$ and $\ctrlpol_{\text{dac}}$  until a plan $\plan$ is found.
\end{theorem}
\begin{proof}
Let $\ctrlpol_{\text{aac}}$ be an adaptive algorithm configuration policy.
Now, we consider the family of planning tasks $\ptask_n$ (Figure \ref{fig:proof_task}) with $|\ptask_{n}| = O(n)$ and a collection of two heuristics $\heus = \{h_0,h_1\}$. 
The heuristic estimates of $h_0$ and $h_1$ are shown in Figure \ref{fig:proof_task} and the open lists of \gbf{} at each time step $t$ are visualized in Figure \ref{fig:proof_open}.
In time step $0$, it is irrelevant which heuristic is selected, always leading to time step $1$, where state $s_3$ is the most promising state according to heuristic $h_0$, while state $s_1$ is the most promising state according to heuristic $h_1$.
In time step $1$, $\ctrlpol_{\text{aac}}$ can either select heuristic $h_0$ or $h_1$. 
We first assume that $\ctrlpol_{\text{aac}}$ selects $h_0$ so that state $s_3$ is expanded, leading to exponentially many states $s_k$, which are all evaluated with $h_0(s_k) = h_1(s_k) = 1$ and thus are all expanded before $s_1$.
Therefore, the unique goal state $s_2$ is found after all other states in the state space $\states$ have been expanded. 

In comparison, for $\ctrlpol_{\text{dac}}$ we can pick the policy that always selects the heuristic with minimum average heuristic value of all states in the corresponding open list, i.e., $\argmin_{h \in \heus} \mu_h$.
Following $\ctrlpol_{\text{dac}}$, first $h_0$ and then $h_1$ is selected, generating the goal state $s_2$ in time step $1$.
Therefore, $\ctrlpol_{\text{dac}}$ only expands $2$ states, while $\ctrlpol_{\text{aac}}$ expands $2^{n-2}$ states until a goal state is found. 

Finally, for a policy $\ctrlpol_{\text{aac}}$ that selects $h_1$ at time step $1$, it is possible to swap the heuristic estimates of $h_0$ and $h_1$ in the constructed collection of heuristics, resulting in the same number of state extensions.
\end{proof}

\begin{figure}
    \centering
    \resizebox{0.9\columnwidth}{!}{
        \begin{tikzpicture}
\draw node[align=center,label={[name=mylabel,label distance=-2mm]above:$h_0$},
    append after command={[rounded corners]
(a.north-|mylabel.west)-|(a.west)|-(a.south)-|(a.east)|-(a.north-|mylabel.east)},
] 
(a) at(0,0) {\\\\\\$\langle s_0, 5 \rangle$};

\draw node[align=center,label={[name=mylabel,label distance=-2mm]above:$h_1$},
    append after command={[rounded corners]
(b.north-|mylabel.west)-|(b.west)|-(b.south)-|(b.east)|-(b.north-|mylabel.east)},
] 
(b) at(0,-2.35) {\\\\\\$\langle s_0, 6 \rangle$};

\draw node[align=center,label={[name=mylabel,label distance=-2mm]above:$h_0$},
    append after command={[rounded corners]
(c.north-|mylabel.west)-|(c.west)|-(c.south)-|(c.east)|-(c.north-|mylabel.east)},
] 
(c) at(2.5,0) {\\\\$\langle s_3, 3 \rangle$\\$\langle s_1, 5 \rangle$};

\draw node[align=center,label={[name=mylabel,label distance=-2mm]above:$h_1$},
    append after command={[rounded corners]
(d.north-|mylabel.west)-|(d.west)|-(d.south)-|(d.east)|-(d.north-|mylabel.east)},
] 
(d) at(2.5,-2.35) {\\\\$\langle s_1, 3 \rangle$\\$\langle s_3, 4 \rangle$};

\draw node[align=center,label={[name=mylabel,label distance=-2mm]above:$h_0$},
    append after command={[rounded corners]
(e.north-|mylabel.west)-|(e.west)|-(e.south)-|(e.east)|-(e.north-|mylabel.east)},
] 
(e) at(5,0) {\\$\langle s_k,1 \rangle$\\$\cdots$\\$\langle s_1, 5 \rangle$};

\draw node[align=center,label={[name=mylabel,label distance=-2mm]above:$h_1$},
    append after command={[rounded corners]
(f.north-|mylabel.west)-|(f.west)|-(f.south)-|(f.east)|-(f.north-|mylabel.east)},
] 
(f) at(5,-2.35) {\\$\langle s_k,1 \rangle$\\$\cdots$\\$\langle s_1, 3 \rangle$};

\draw node[align=center,label={[name=mylabel,label distance=-2mm]above:$h_0$},
    append after command={[rounded corners]
(g.north-|mylabel.west)-|(g.west)|-(g.south)-|(g.east)|-(g.north-|mylabel.east)},
] 
(g) at(7.5,0) {\\\\$\mathbf{\langle s_2, 0 \rangle}$\\$\langle s_3, 3 \rangle$};

\draw node[align=center,label={[name=mylabel,label distance=-2mm]above:$h_1$},
    append after command={[rounded corners]
(h.north-|mylabel.west)-|(h.west)|-(h.south)-|(h.east)|-(h.north-|mylabel.east)},
] 
(h) at(7.5,-2.35) {\\\\$\mathbf{\langle s_2, 0 \rangle}$\\$\langle s_3, 4 \rangle$};


\node (e1) at ($(a.north)+(1,0.75)$) {$\ctrlpol_{as},\ctrlpol_{aac}$: $h_0$};
\draw[] ($(a.north)+(-0.25,0.45)$) |- (e1);
\draw[-{Latex[scale=1.0]}] (e1) -| ($(c.north)+(-0.3,0.45)$);

\node (e2) at ($(c.north)+(1.4,0.75)$) {$\ctrlpol_{as},\ctrlpol_{aac}$: $h_0$};
\draw[] ($(c.north)+(0.25,0.45)$) |- (e2);
\draw[-{Latex[scale=1.0]}] (e2) -| ($(e.north)+(0.1,0.45)$);

\node (e3) at ($(a.south)+(1,-2.95)$) {$\ctrlpol_{dac}$: $h_0$};
\draw[] ($(a.south)+(-0.25,-2.65)$) |- (e3);
\draw[-{Latex[scale=1.0]}] (e3) -| ($(c.south)+(-0.3,-2.65)$);

\node (e4) at ($(e.south)+(0,-2.95)$) {$\ctrlpol_{dac}$: $h_1$};
\draw[] ($(c.south)+(0.25,-2.65)$) |- (e4);
\draw[-{Latex[scale=1.0]}] (e4) -| ($(g.south)+(-0.1,-2.65)$);

\draw[dashed] (-1,1.05) rectangle (1,-3.3) node[pos=.5] {Step: $0$};
\draw[dashed] (-1+2.5,1.05) rectangle (1+2.5,-3.3) node[pos=.5] {Step: $1$};
\draw[dashed] (-1+5,1.05) rectangle (1+5,-3.3) node[pos=.5] {Step: $2a$};
\draw[dashed] (-1+7.5,1.05) rectangle (1+7.5,-3.3) node[pos=.5] {Step: $2b$};

\end{tikzpicture}
    }
    \caption{Visualization of two heuristics used to solve an instance of the planning task family $\ptask{}_n$.\label{fig:proof_open}}
\end{figure}
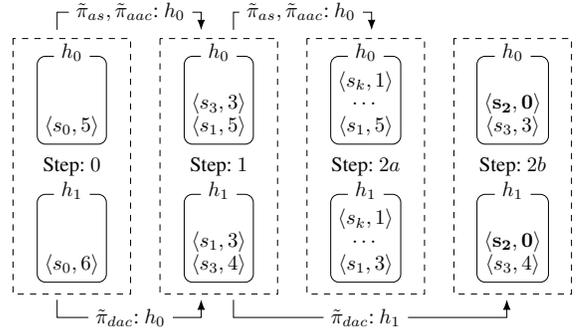

\begin{theorem}\label{thm:exp_expansion_number_as}
For each algorithm selection policy $\ctrlpol_{\text{as}}$ there exists a family of planning instances $\ptask_n$, a collection of heuristics $\heus$ and a dynamic control policy $\ctrlpol_{\text{dac}}$, so that \gbf{} with $\heus$ and $\ctrlpol_{\text{as}}$ expands exponentially more states in $|\ptask_n|$ than \gbf{} with $\heus$ and $\ctrlpol_{\text{dac}}$  until a plan $\plan$ is found. \QED
\end{theorem}
For the proof of this theorem, we refer to the longer arXiv version of this paper \citep{speck-et-al-arxiv2020}.

In Theorems \ref{thm:exp_expansion_number_aac} and \ref{thm:exp_expansion_number_as} we assume for simplicity that expanded states are directly removed from all open lists.  
In practice, open lists are usually implemented as min-heaps, and it is costly to search and remove states immediately.
Thus, states that have already been expanded are kept in the open lists and ignored as soon as they have reached the top. 
We note that this does not affect the theoretical results.

Finally, we want to emphasize that all results presented are theoretical and based on the assumption that it is possible to learn good dynamic control policies. 
Next, we show that it is possible in practice to learn such dynamic control policies.


\begin{table*}[tbp]
    \centering
    \resizebox{0.825\textwidth}{!}{
        \begin{tabular}{lrrrrrrr |rrc}
    \toprule
     Algorithm & \multicolumn{3}{c}{\name{control policy}} & \multicolumn{4}{c}{\name{single heuristic}} & \multicolumn{3}{|c}{\name{best as (oracle)}} \\
     \cmidrule(lr){2-4} \cmidrule(lr){5-8} \cmidrule(lr){9-11} 
     Domain (\# Inst.) & \multicolumn{1}{c}{\textbf{\name{rl}}} & \multicolumn{1}{c}{\name{rnd}} & \multicolumn{1}{c}{\name{alt}} & \multicolumn{1}{c}{$h_{\text{ff}}$} & \multicolumn{1}{c}{$h_{\text{cg}}$} & \multicolumn{1}{c}{$h_{\text{cea}}$} & \multicolumn{1}{c}{$h_{\text{add}}$} & 
     \multicolumn{1}{|c}{\textbf{\name{rl}}} &
     \multicolumn{1}{c}{\name{alt}} &
     \multicolumn{1}{c}{\name{single $h$}} \\
     \cmidrule(lr){1-1} \cmidrule(lr){2-8} \cmidrule(lr){9-11}
     \rowcolor{Gray}
     \name{barman} (100) & \textbf{84.4} & 83.8 & 83.3 & 66.0 & 17.0 & 18.0 & 18.0 & \textbf{89.0} & 84.0 & 67.0  \\
     \name{blocksworld} (100) & \textbf{92.9} & 83.6 &83.7 & 75.0 & 60.0 & 92.0 & 92.0 & \textbf{96.3} & 88.0 & 93.0 \\
     \rowcolor{Gray}
     \name{childsnack} (100) & \textbf{88.0} & 86.2 & 86.7 & 75.0 & 86.0 & 86.0 & 86.0 & \textbf{88.0} & \textbf{88.0} & 86.0 \\
     \name{rovers} (100) & 95.2 & \textbf{96.0} & \textbf{96.0} & 84.0 & 72.0 & 68.0 & 68.0 & \textbf{96.0} & \textbf{96.0} & 91.0 \\
     \rowcolor{Gray}
     \name{sokoban} (100) & 87.7 & 87.1 & 87.0 & 88.0 & \textbf{90.0} & 60.0 & 89.0 & 88.6 & 87.0 & \textbf{92.0} \\
     \name{visitall} (100) & 56.9 & 51.0 & 51.5 & 37.0 & \textbf{60.0} & \textbf{60.0} & \textbf{60.0} & \textbf{61.4} & 52.0 & 60.0 \\
     \cmidrule(lr){1-1} \cmidrule(lr){2-8} \cmidrule(lr){9-11}
     \textsc{sum} (600) & \textbf{505.1} & 487.7 & 488.2 & 425.0 & 385.0 & 384.0 & 413.0 & \textbf{519.3} & 495.0 & 489.0 \\
     \bottomrule
\end{tabular}
 
    }
    \caption{Average coverage of different policies for the selection of a heuristic in each expansion step when evaluating the strategies on the prior unseen \emph{test} set. 
    The first three columns are control policies, the next four are individual heuristic searches, while the last three represent the \emph{best algorithm selection} of the corresponding strategies, i.e., oracle selector for each instance.}
    \label{tab:real_coverage}
\end{table*}

\section{Empirical Evaluation}

We conduct experiments\footnote{Resources: \url{https://github.com/speckdavid/rl-plan}
} to measure the performance of our reinforcement learning (\name{rl}) approach on domains of the International Planning Competition (IPC).
For each domain, the \name{rl} policies are trained on a training set and evaluated on a disjoint prior unseen test set of the same domain obtained by a random split.
Note that the policies we consider here are not domain-independent, although it is generally possible to add instance- and domain-specific information to the state features. 
We leave the task of learning domain-independent policies for future work.

\subsection{Setup}

All experiments are conducted with \name{Fast Downward} \cite{helmert-jair2006} as the underlying planning system. 
We use (``eager'') greedy best-first search \cite{richter-helmert-icaps2009} and min-heaps to represent the open lists \cite{roeger-helmert-icaps2010}.
Although there are many more sophisticated search strategies and components, we choose a vanilla search strategy to reduce the factors that might affect the comparison of the actual research question of whether the learned \name{dac} policies can improve the search performance of heuristic search.
Nevertheless, our approach is in principle also capable of handling more complex search strategies and components, such as lazy eager search with preferred operators and simple handcrafted \name{dac} like policies such as boosting heuristics \cite{richter-et-al-ipc2011}.

We implemented an extension for \name{Fast Downward}, which makes it possible to communicate with a controller (dynamic control policy) via TCP/IP and thus to send relevant information (state features and reward) in each time/expansion step and to receive the selected parameter (heuristic). 
This architecture allows the planner and controller to be decoupled, making it easy to replace components.
We considered four different heuristic estimators as configuration space, i.e., $\tilde{\Theta} = H = \{h_{\text{ff}},h_{\text{cg}}, h_{\text{cea}},h_{\text{add}}\}$ which can be changed at each time step:
\begin{inparaenum}[(1)]
    \item the FF heuristic $h_{\text{ff}}$ \cite{hoffmann-nebel-jair2001},
    \item the causal graph heuristic $h_{\text{cg}}$ \cite{helmert-icaps2004},
    \item the context-enhanced additive heuristic $h_{\text{cea}}$ \cite{helmert-geffner-icaps2008}, and
    \item the additive heuristic $h_{\text{add}}$ \cite{bonet-geffner-aij2001}.\footnote{We also conducted additional experiments with five heuristics instead of four, including the $h_{\text{lm-count}}$ heuristic \cite{richter-et-al-aaai2008}; please see \techreport{}.
}
\end{inparaenum}

For evaluation (final planning runs) we used a maximum of 4 GB memory and 5 minutes runtime.
All experiments were run on a compute cluster with nodes equipped with two Intel Xeon Gold 6242 32-core CPUs, 20 MB cache and 188 GB shared RAM running Ubuntu 18.04 LTS 64 bit. 

Similar to \citeN{biedenkapp-et-al-ecai2020}, we use $\epsilon$-greedy deep Q-learning in the form of a double DQN~\cite{hasselt-et-al-aaai2016} implemented in \name{chainer} \cite{tokui-et-al-kdd2019} (\name{chainerRL} v0.7.0) to learn the dynamic control policies. 
The networks are trained using \name{Adam}\footnote{We use \name{chainer}'s v0.7.0 default parameters for \name{Adam}.} \cite{kingma-et-al-arxiv2015} for $10^6$ update steps on a single machine of our cluster with two CPU cores and 20 GB RAM.
We use a cutoff of $7\,500$ control/expansion steps in order to avoid policies being executed arbitrarily long during training. 
Complex instances may not be solved within this cutoff, even with the optimal policy, and thus learning occurs on simpler instances.
However, the underlying assumption is that well performing policies for smaller instances generalize to larger instances within a domain. 
To determine the quality of a learned policy, we evaluated it every $30\,000$ steps during training and save the best policy we have seen so far. 
In total, we performed $5$ independent runs of our control policies for each domain, for which we report the average performance. The policies are represented by neural networks for which we determined the hyperparameters in a white-box experiment on a new artificial domain (see \techreport{}) and kept these hyperparameters fixed for all experiments.

\subsection{Experiments}
We evaluated the performance of our \name{rl} approach on six domains of the International Planning Competition (IPC). 
These domains were chosen because there are instance generators available online\footnote{\url{https://github.com/AI-Planning/pddl-generators}} that make it possible to create a suitable number of instances of different sizes. 
Furthermore, instances of these domains usually require a significant number of state expansions in order to find a plan. 
For this purpose, we generated $200$ instances for all domains and randomly divided them into disjoint training and test sets with the same size of $100$ instances each.
For each domain we trained five dynamic control policies on the training set and compared them with other approaches on the unseen test set. We are mainly interested in comparing different policies for heuristic selection, which is why, here, the planner always maintains all four open lists, even if only one heuristic is used, and the controller, i.e., the dynamic control policy, alone decides which heuristic is selected. 

Table \ref{tab:real_coverage} shows the percentage of solved instances per domain, i.e., the average coverage, on the test set.
Each domain has a score in the range of $0$-$100$, with larger values indicating more solved instances on average. 
More precisely, it is possible to obtain a score between $0$ and $1$ for each planning instance.
A value of 0 means that the instance was never solved by the approach, $0.5$ means that the instance was solved in half the runs, and $1$ means that the instance was always solved.
These scores are added up to give the average coverage per domain. 

The first three columns correspond to control policies. 
Entry \name{rl} is the average coverage of the five trained dynamic control policies based on reinforcement learning, each averaging over $25$ runs with different seeds. 
Entry \name{rnd} denotes the average coverage of $25$ runs, where a random heuristic is selected in each step. 
Entry \name{alt} stands for the average over all possible permutations of the execution of alternation. 
Note that there are $4!=24$ different ways of executing alternation with four different heuristics. 
The \name{single heuristic} columns show the coverage when only the corresponding heuristic is used. 
Finally, the columns for selecting the \emph{best} algorithm selection (\textsc{best as}) stand for the use of an oracle selector, which selects the best configuration of the corresponding technique for each instance. 
In other words, the best algorithm selection for \name{rl} is to choose the best dynamic control policy from the five trained policies for each instance, the best algorithm selection for \name{alt} is to choose the best permutation of alternation for each instance and the best algorithm selection for \name{single $h$} is to choose the best heuristic for each instance. 

\paragraph{Coverage.} The results of Table \ref{tab:real_coverage} show that our approach (\name{rl}) performs best on average in terms of coverage (individual coverage of the five policies: $505.4$, $500.6$, $501.6$, $507.4$, $510.1$).    
\name{alt} is slightly better than the uniform randomized choice of a heuristic \name{rnd}, which indicates that the most important advantage of \name{alt} is to use each heuristic equally with frequent switches and not to switch between them systematically. 
Furthermore, consistent with the results of \citeN{roeger-helmert-icaps2010}, single heuristics perform worse than the use of multiple heuristics. 
Interestingly, in the domain \name{visitall}, single heuristics have the highest coverage and while \name{rnd} and \name{alt} have a low coverage, \name{rl} performs better. 
This indicates that in this domain, the dynamic control policies of \name{rl} were able to infer that a static policy is well performing or to exclude certain single heuristics. 
In \name{blocksworld}, \name{rl} has the highest coverage among all approaches. 
A possible explanation is that a dynamic policy is the key to solving difficult instances in this domain. 
This assumption is supported by the observation that the best algorithm selection, i.e., the oracle selection of \name{rl}, clearly exceeds the other approaches in \name{blocksworld}. 
Finally, in \name{rover}, the use of multiple heuristics seems to be important, and while \name{rl} scores better than using single heuristics, the learned policy scores worse than \name{rnd} and \name{alt}. 
This may be due to overfitting which we will discuss below.
We also compare our approach to the theoretically best possible algorithm selector. Considering the columns \textsc{Best AS}, we observe that oracle single heuristic selection and oracle alternating selection do not perform better than the average performance of our learned \name{rl} policies, which shows that 1) heuristic search with multiple heuristics can in practice benefit from dynamic algorithm configuration and 2) it is possible to learn well performing dynamic policies domain-wise. Even under the unrealistic circumstances of an \emph{optimal} algorithm selector, our learned policies perform better and therefore outperform all possible algorithm selection policies.

If we increase the configuration space by adding another heuristic ($h_{\text{lm-count}}$) the overall coverage of the control policies increases. 
However, the results are still qualitatively similar to those presented here, showing that the learned \name{dac} policies perform best overall (see \techreport{}).

We also want to mention the computational overhead of our \name{rl} approach compared to \name{alt} and \name{single heuristic} search approaches. 
While the performance of \name{rl} on the test set still exceeds the \name{single heuristic} search of \name{Fast Downward} for all four heuristics with a \emph{single} open list (maintaining only the used heuristic), \name{rl} performs slightly worse than the internal heuristic alternation strategy of \name{Fast Downward}. 
In the future, the overhead can be reduced by integrating the reinforcement learning part directly in \name{Fast Downward} instead of communicating via TCP/IP. 

\begin{table}[tbp]
    \centering
    \subfloat[Test set]{
    \resizebox{0.95\columnwidth}{!}{
        \begin{tabular}{lrrrrrrr}
    \toprule
     Algorithm & \multicolumn{3}{c}{\name{control policy}} & \multicolumn{4}{c}{\name{single heuristic}} \\
     \cmidrule(lr){2-4} \cmidrule(lr){5-8}
     Metric & \multicolumn{1}{c}{\textbf{\name{rl}}} & \multicolumn{1}{c}{\name{rnd}} & \multicolumn{1}{c}{\name{alt}} & \multicolumn{1}{c}{$h_{\text{ff}}$} & \multicolumn{1}{c}{$h_{\text{cg}}$} & \multicolumn{1}{c}{$h_{\text{cea}}$} & \multicolumn{1}{c}{$h_{\text{add}}$} \\
     \midrule
     \rowcolor{Gray}
     \textsc{coverage} & \textbf{84.2} & 81.3 & 81.4 & 70.8 & 64.2 & 64.0 & 68.8 \\
     \textsc{guidance} & \textbf{38.5} & 37.4 & 37.5 & 30.8 & 27.6 & 28.6 & 30.4 \\
     \rowcolor{Gray}
     \textsc{speed} & \textbf{66.6} & 62.8 & 62.8 & 54.9 & 50.4 & 50.3 & 54.0 \\
     \textsc{quality} & \textbf{76.2} & 76.0 & 76.0 & 65.8 & 57.6 & 56.2 & 60.9 \\
     \bottomrule
\end{tabular}
    }
    \label{tab:real_metrics_test}}
    
    \subfloat[Training set]{
    \resizebox{0.95\columnwidth}{!}{
        \begin{tabular}{lrrrrrrr}
    \toprule
     Algorithm & \multicolumn{3}{c}{\name{control policy}} & \multicolumn{4}{c}{\name{single heuristic}} \\
     \cmidrule(lr){2-4} \cmidrule(lr){5-8}
     Metric & \multicolumn{1}{c}{\textbf{\name{rl}}} & \multicolumn{1}{c}{\name{rnd}} & \multicolumn{1}{c}{\name{alt}} & \multicolumn{1}{c}{$h_{\text{ff}}$} & \multicolumn{1}{c}{$h_{\text{cg}}$} & \multicolumn{1}{c}{$h_{\text{cea}}$} & \multicolumn{1}{c}{$h_{\text{add}}$} \\
     \midrule
     \rowcolor{Gray}
     \textsc{coverage} & \textbf{87.0} & 83.6 & 83.0 & 71.7 & 64.3 & 65.0 & 68.5 \\
     \textsc{guidance} & \textbf{39.8} & 38.3 & 38.4 & 31.4 & 26.6 & 28.8 & 30.2 \\
     \rowcolor{Gray}
     \textsc{speed} & \textbf{69.3} & 65.3 & 65.4 & 56.0 & 49.1 & 51.1 & 54.2 \\
     \textsc{quality} & \textbf{79.5} & 77.9 &  77.5 & 66.8 & 57.3 & 58.0 & 61.3 \\
     \bottomrule
\end{tabular}
    }
    \label{tab:real_metrics_train}}
    \caption{A comparison of different control policies and single heuristic search measuring coverage, guidance, speed and solution quality on the prior unseen \emph{test} set \protect\subref{tab:real_metrics_test} and the \emph{training} set  \protect\subref{tab:real_metrics_train}. Higher scores are preferable for all metrics.}
    \label{tab:real_metrics}
\end{table}

\paragraph{Guidance, speed and quality.} Table \ref{tab:real_metrics} shows four different metrics including the coverage from above. We additionally evaluate the guidance, speed and quality for each approach with a rating scale \cite{richter-helmert-icaps2009,roeger-helmert-icaps2010}. 
For \emph{guidance}, tasks solved within one state expansion get one point, while unsolved tasks or tasks solved with more than $10^6$ state expansions get zero points. 
Between these extremes the scores are interpolated logarithmically. 
For \emph{speed} the algorithm gets one point for tasks solved within one second, while the algorithm gets zero points for unsolved tasks or tasks solved in $300$ seconds. 
For \emph{quality} the algorithm gets a score of $c^* /c$ for a solved task, where $c$ is the cost of the reported plan and $c^*$ is the cost of the best plan found with any approach. 
Finally, the sum of each metric is divided by the number of domains to obtain a total score between $0$ and $100$.
Considering those metrics, control policies perform better than single heuristic approaches. Furthermore, dynamic control polices obtained by \name{rl} perform best according to all metrics. 
However, this analysis favors approaches which solve more instances than others. 
Recall that plan quality is not taken into account when learning a policy, which explains the small advantage of \name{rl} in plan quality, even though more instances have been solved by \name{rl}.

\begin{figure}
\centering
    \subfloat[Policy similarity]{
    \resizebox{0.85\columnwidth}{!}{
        \includegraphics[width=.45\textwidth]{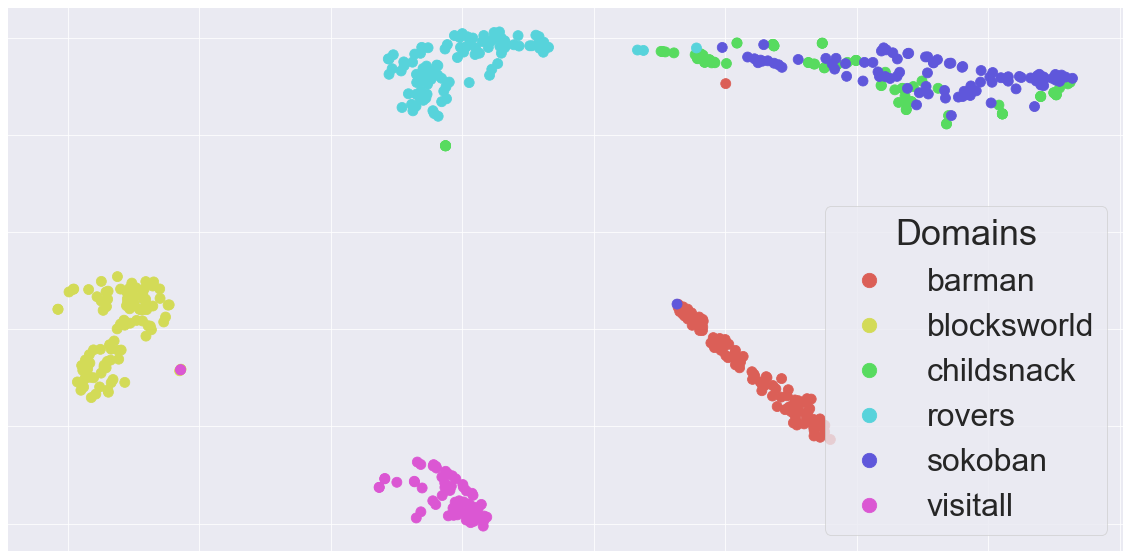}
    }
    \label{fig:h_domain}}
    
    \subfloat[Switching frequency]{
    \resizebox{0.85\columnwidth}{!}{
        \includegraphics[width=.45\textwidth]{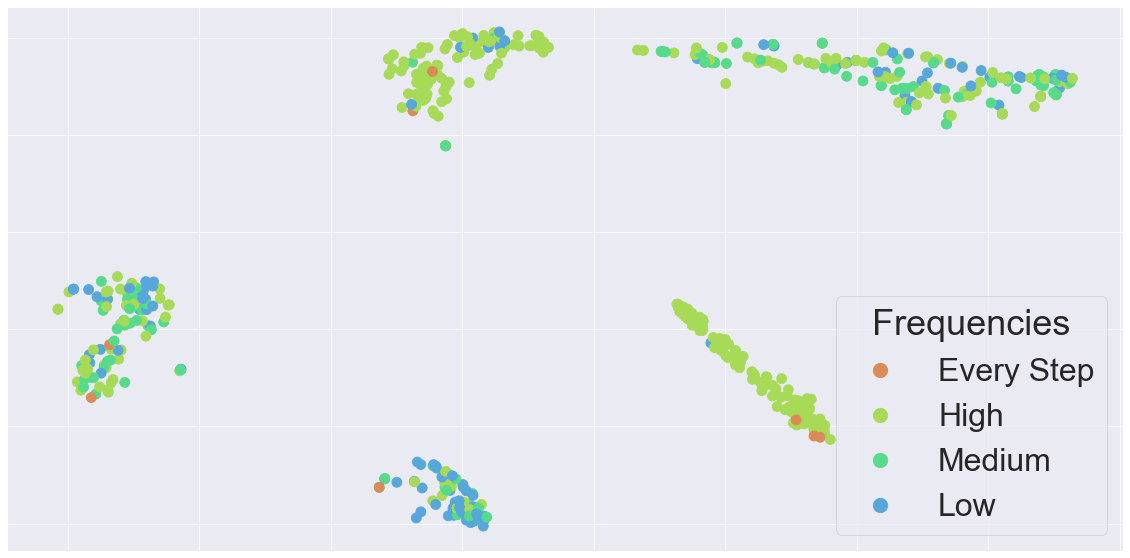}
    }
    \label{fig:h_switching}}
    \caption{Similarity of the learned \name{rl} policies using t-SNE \cite{hinton-roweis-nips2002}. 
    Each point represents a run of a policy on a different instance of the test set. 
    The distances between the points represents the similarity of the policies heuristic usage.
    }
    \label{fig:h_analysis}
\end{figure}

\paragraph{Policy analysis.}
We analyzed the resulting policies on the test set considering all successful runs, i.e., runs where a plan was found. 
Our first finding is that the \name{rl} policies often favor one of the heuristics over all others. 
In \name{barman}, \name{childsnack}, \name{rovers}, and \name{visitall}, the $h_{\text{ff}}$ heuristic is preferred, while in \name{blocksworld} the $h_{\text{add}}$ and in \name{visitall} the $h_{\text{cea}}$ is most often chosen. 
Interestingly, the percentage of use of the preferred heuristic varies greatly by domain. 
In \name{sokoban}, e.g., $h_{\text{ff}}$ is selected for $98\%$ expansion steps on average, while in \name{Barman}, the policies only select this
heuristic $60\%$ of the time on average. 

This raises the question of how similar the learned policies are and whether the policies are dynamic in the sense that they often switch between different heuristics.
To visualize the policies, we used t-distributed stochastic neighbor embedding \citep[t-SNE;][]{hinton-roweis-nips2002}, which maps higher-dimensional data into a 2-D space for visualization (Figure \ref{fig:h_analysis}). 
As input, we used a vector of four numbers for each instance, indicating the percentage of times a heuristic was chosen, and mapped it to a space where, informally, the more similar the policy executions are, the closer the instances are to each other.
Figure \ref{fig:h_domain} shows the similarity of the policy executions for each instance of the different domains. 
It can be observed that the overall heuristic selection of the learned policies is similar domain-wise, given the resulting clustering. 
Figure \ref{fig:h_switching} additionally visualizes how often a heuristic was switched, i.e., the frequency of switching the heuristic between subsequent expansion steps. 
Interestingly, in the majority of the instances, the switching frequency is high (sequences with no $h$ switch $< 100$ steps), while in the instances corresponding, for example, to the domain \name{sokoban}, the switching frequency is low (sequences with no $h$ switch $> 1\,000$ steps). 
This correlates with the fact that in this domain a single heuristic was selected with a high percentage.

We conclude that our approach is able to learn highly dynamic policies, but also highly static policies, depending on the problem instance at hand. 
Intuitively, this makes sense, since there may be domains for which a static policy performs best. 
However, \name{dac} is a generalization of algorithm selection which allows to learn such static policies as well.

\paragraph{Training performance.} We compare the performance of our \name{rl} approach on the \emph{training} set (Table \ref{tab:real_metrics_train}) with the performance of \name{rl} on the test set (Table \ref{tab:real_metrics_test}). Interestingly, \name{rl} performs better on the \emph{training} set which can be attributed to a certain degree of overfitting and can explain why in some instances the performance of \name{rl} is worse than other approaches on the \emph{test} set. 
This issue can be addressed by tuning the hyperparameters, expanding the training set or adding further state features.
Overall, the improvements DAC yields over the other methods more than outweigh any overfitting, leading to DAC performing best on the test set.

\section{Conclusion}
We investigated the use of dynamic algorithm configuration for planning. 
More specifically, we have shown that dynamic algorithm configuration can be used for dynamic heuristic selection that takes into account the internal search dynamics of a planning system. 
Dynamic policies for heuristic selection generalize policies of existing approaches like algorithm selection and adaptive algorithm control, and they can improve search performance exponentially. 
We presented an approach based on dynamic algorithm configuration and showed empirically that it is possible to learn policies capable of outperforming other approaches in terms of coverage.

In future work, we will investigate domain-specific state features to learn domain-independent dynamic policies. 
Further, it is possible to dynamically control several parameters of a planner and to switch dynamically between different search algorithms.
This raises the question how the search progress \cite{aine-likhackev-icaps2016} can be shared when using different search strategies. 
In particular, if we want to combine different search 
techniques, such as heuristic search \cite{bonet-geffner-aij2001}, symbolic search \cite{torralba-et-al-aij2017,speck-et-al-icaps2018} and planning as satisfiability \cite{kautz-selman-ecai1992,rintanen-aij2012}, it is an open question how to share the search progress.

\section{Acknowledgments}
D. Speck was supported by the German Research Foundation
(DFG) as part of the project EPSDAC (MA 7790/1-1). M. Lindauer acknowledges support by the DFG under LI 2801/4-1. A. Biedenkapp, M. Lindauer and F. Hutter acknowledge funding by the Robert Bosch GmbH.

\setcounter{secnumdepth}{1} 
\appendix
\renewcommand{\thetable}{\Alph{section}\arabic{table}}
\renewcommand{\thefigure}{\Alph{section}\arabic{figure}}
\setcounter{figure}{0}
\setcounter{theorem}{3}
\section{Theoretical Results}\label{app:theoretical_results}
\begin{theorem}\label{thm:exp_expansion_number_as_2}
For each algorithm selection policy $\ctrlpol_{\text{as}}$ there exists a family of planning instances $\ptask_n$, a collection of heuristics $\heus$ and a dynamic control policy $\ctrlpol_{\text{dac}}$, so that \gbf{} with $\heus$ and $\ctrlpol_{\text{as}}$ expands exponentially more states in $|\ptask_n|$ than \gbf{} with $\heus$ and $\ctrlpol_{\text{dac}}$  until a plan $\plan$ is found.
\end{theorem}
\begin{proof}
Let $\ctrlpol_{\text{as}}$ be an algorithm selection policy. 
We consider the family of planning tasks $\ptask'_n$, which is similar to the family of planning tasks $\ptask_n$ (Figure \ref{fig:proof_task} in the main paper), with one modification: the goal state $s_2$ is not directly reachable from $s_1$, but via an additional state $s'$. 
In other words, we insert the state $s'$ between $s_1$ and $s_2$. 
Furthermore, we again consider a collection of two heuristics $\heus = \{h_0,h_1\}$ with the heuristic estimates shown in Figure \ref{fig:proof_task} (main paper) and $h_0(s')=2$ and $h_1(s')=10$.
The idea is that both heuristics alone lead to the expansion of exponentially many states, whereas a dynamic switch of the heuristic only leads to constantly many expansions.

Policy $\ctrlpol_{\text{as}}$ selects exactly one heuristic, $h_0$ or $h_1$, for each planning task.
If $h_0$ is selected, with the same argument used in the proof of Theorem \ref{thm:exp_expansion_number_aac} (main paper), exponentially many states in $|\ptask'_n|$ are expanded. 
If $h_1$ is selected, in time step 2, states $s_3$ and $s'$ are contained in both open lists. 
According to $h_1$, state $s_3$ is more promising than $s'$, which leads again to an expansion of exponentially many states in $|\ptask'_n|$. 

In comparison, for $\ctrlpol_{\text{dac}}$ we pick again the policy that always selects the heuristic with minimum average heuristic value of all states in the corresponding open list, i.e. $\argmin_{h \in \heus} \mu_h$. 
Policy $\ctrlpol_{\text{dac}}$ selects first $h_0$, followed by $h_1$ and again $h_0$, resulting in the generation of the goal state after three state extensions. 
\end{proof}

\section{White-Box Experiments}\label{app:white}
We conducted preliminary experiments on a newly created \name{artificial} domain with two artificial heuristics. 
This domain is designed so that in each step, only one of two heuristics is informative. 
In other words, similar to the constructed example in the proof of Theorem \ref{thm:exp_expansion_number_as}, at each time step, only one heuristic leads to the expansion of a state which is on the shortest path to a goal state.
In order to obtain a good control policy that leads to few state expansions, it is necessary to derive a \emph{dynamic} control policy from the state features.
We generated $30$ training instances on which we performed a small grid search over the following parameters \emph{\#layers} $\in \left\{2,5\right\}$, \emph{hidden units} $\in \left\{50, 75, 150, 200\right\}$ and \emph{epsilon decay} $\in \left\{2.5 \times 10^5, 5 \times 10^5\right\}$.
We determined that a $2$-layer network with $75$ hidden units and a linear decay for $\epsilon$ over $5 \times 10^5$ steps from $1$ to $0.1$ worked best \footnote{
Note that the hyperparameters for experiments on the IPC domains have not been further tuned.}.

Interestingly, it was possible to learn policies with a performance close to the optimal policy, see Figure~\ref{fig:toy-performance}. Both individual heuristics perform poorly (even when using an oracle selector). Randomly deciding which heuristic to play performs nearly as good as the alternating strategy that alternates between the heuristics at each step. In the beginning the learned policy needs some time to figure out in which states a heuristic might be preferable. However, it quickly learns to choose the correct heuristic, outperforming all other methods and nearly recovering the optimal policy.

\begin{figure}
    \centering
    \includegraphics[width=.43\textwidth]{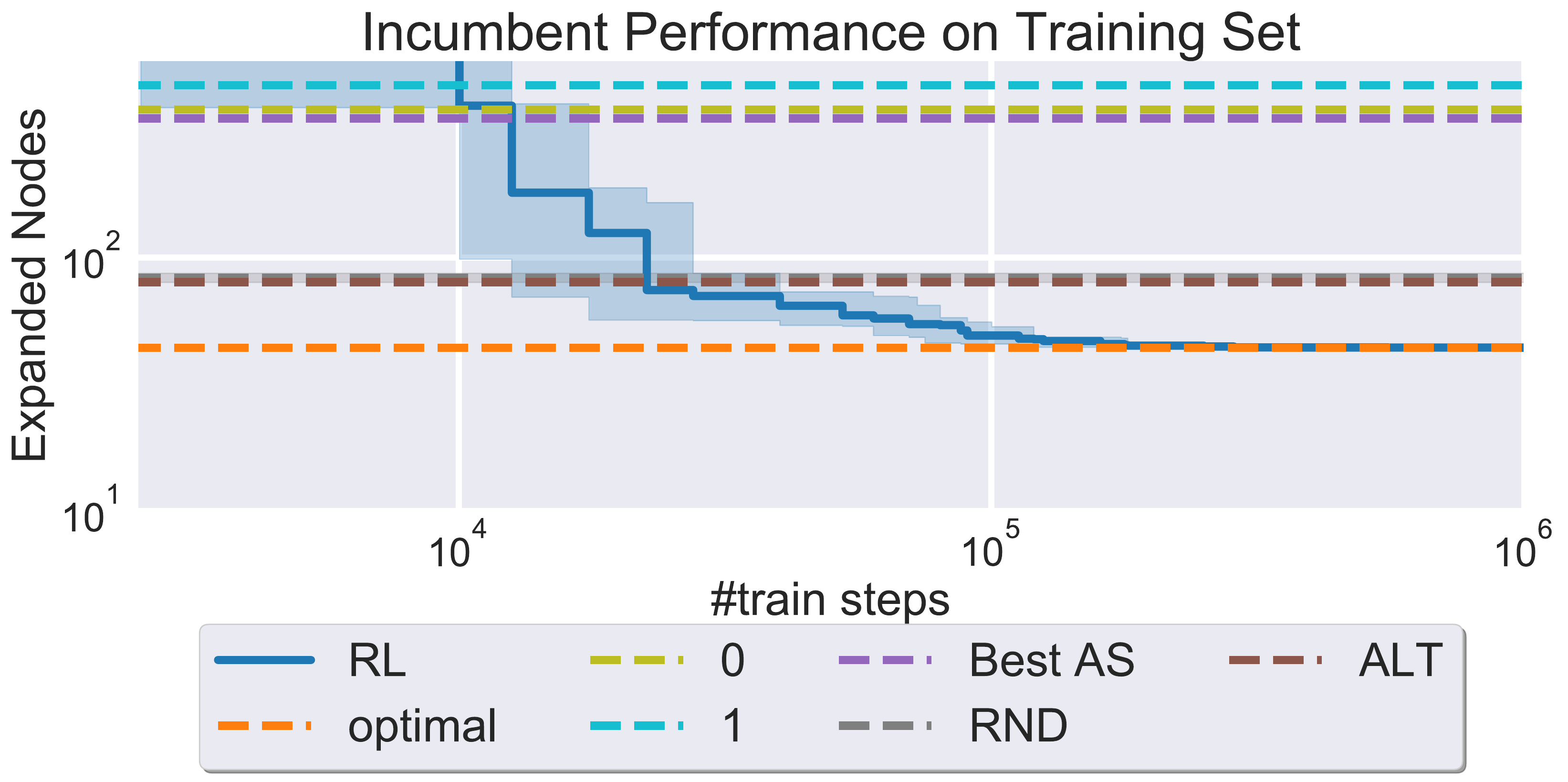}
    \caption{Performance of the best learned policy during training (\name{rl}), compared to the the performance of the individual heuristics (\name{0} \& \name{1}), the oracle selector (\name{best as}), an alternating schedule (\name{alt}), a random policy (\name{rnd}) and the optimal policy. Dashed lines indicate the performance of our baselines, the solid line the mean performance and the shaded area the standard deviation of our approach.
    }
    \label{fig:toy-performance}
\end{figure}

\section{Domain Dependence}
We limit ourselves to learning domain-dependent policies as it was uncertain if learned policies would be able to transfer from training with a cutoff to the test set with potentially much longer trajectories.
In particular, we used a conservative cutoff of $7500$ node expansions for training.
When evaluating our fully trained policies, we did not use this node expansion cutoff, but ran \name{Fast Downward} for $300$ seconds (allowing for potentially many more than $7500$ node expansions).
To gain insights into how many of the training and testing problems could be solved by our learned policies without use of the conservative cutoff we evaluated all our learned policies with the cutoff of $300$ seconds.
On the \textit{training} set roughly $31\%$ of all training instances that can be solved within $300$ seconds by our policies require more than $7500$ node expansions.
On the test set it increases slightly to $33\%$ of the solved instances require more than $7500$ expansions.
It is worth noting that on the test sets for \name{barman} and \name{visitall} this percentage is even $49\%$ and $66\%$ respectively.
As this generalization to unseen instances with potentially much longer trajectories was a first significant hurdle to overcome we limited ourselves to domain-dependent policies.
In future work we will remove this limitation.

\section{Additional Experiments}\label{appendix:sec:experimentes}

\begin{table*}[t]
    \centering
    \resizebox{0.9\textwidth}{!}{
        \begin{tabular}{lrrrrrrrr |rrc}
    \toprule
     Algorithm & \multicolumn{3}{c}{\name{control policy}} & \multicolumn{5}{c}{\name{single heuristic}} & \multicolumn{3}{|c}{\name{best as (oracle)}} \\
     \cmidrule(lr){2-4} \cmidrule(lr){5-9} \cmidrule(lr){10-12} 
     Domain (\# Inst.) & \multicolumn{1}{c}{\textbf{\name{rl}}} & \multicolumn{1}{c}{\name{rnd}} & \multicolumn{1}{c}{\name{alt}} & \multicolumn{1}{c}{$h_{\text{ff}}$} & \multicolumn{1}{c}{$h_{\text{cg}}$} & \multicolumn{1}{c}{$h_{\text{cea}}$} & \multicolumn{1}{c}{$h_{\text{add}}$} & \multicolumn{1}{c}{$h_{\text{lm-count}}$} &
     \multicolumn{1}{|c}{\textbf{\name{rl}}} &
     \multicolumn{1}{c}{\name{alt}} &
     \multicolumn{1}{c}{\name{single $h$}} \\
     \cmidrule(lr){1-1} \cmidrule(lr){2-9} \cmidrule(lr){10-12}
     \rowcolor{Gray}
     \name{barman} (100) & \textbf{85.6} & 84.2 & 84.1 & 64.0 & 16.0 & 18.0 & 18.0 & 77.0 & \textbf{91.2} & 85.0 & 85.0  \\
     \name{blocksworld} (100) & \textbf{95.0} & 89.8 & 90.8 & 75.0 & 60.0 & 92.0 & 92.0 & 75.0 & \textbf{99.0} & 93.0 & 97.0 \\
     \rowcolor{Gray}
     \name{childsnack} (100) & 84.4 & 82.0 & 83.4 & 75.0 & \textbf{86.0} & \textbf{86.0} & \textbf{86.0} & 68.0 & \textbf{89.0} & 85.0 & 86.0 \\
     \name{rovers} (100) & \textbf{100.0} & \textbf{100.0} & \textbf{100.0} & 83.0 & 72.0 & 67.0 & 67.0 & \textbf{100.0} & \textbf{100.0} & \textbf{100.0} & \textbf{100.0} \\
     \rowcolor{Gray}
     \name{sokoban} (100) & 87.0 & 87.1 & 87.0 & 88.0 & \textbf{90.0} & 60.0 & 88.0 & 85.0 & 87.0 & 87.0 & \textbf{92.0} \\
     \name{visitall} (100) & 1\textbf{00.0} & \textbf{100.0} & \textbf{100.0} & 38.0 & 60.0 & 59.0 & 60.0 & \textbf{100.00} & \textbf{100.0} & \textbf{100.0} & \textbf{100.0} \\
     \cmidrule(lr){1-1} \cmidrule(lr){2-9} \cmidrule(lr){10-12}
     \textsc{sum} (600) & \textbf{552.0
} & 543.1 & 545.3 & 423.0 & 384.0 & 382.0 & 411.0 & 505.0 & \textbf{566.2} & 550.0 & 560.0 \\
     \bottomrule
\end{tabular}
 
    }
    \caption{Average coverage of different policies for the selection of a heuristic in each expansion step when evaluating the strategies on the prior unseen \emph{test} set. 
    The first three columns are control policies, the next \emph{five} are individual heuristic searches, while the last three represent the best algorithm selection of the corresponding strategies, i.e., oracle selector for each instance.}
    \label{tab:real_coverage_h5}
\end{table*}
We conducted additional experiments with 5 heuristics instead of 4, also including the $h_{\text{lm-count}}$ heuristic. Overall, the results are similar to those presented in the paper, showing that the learned \name{dac} policy performs best overall. Note that our computer cluster has been updated to Ubuntu 20.04 LTS 64 bit in the meantime. Therefore, we ran all the configurations again, including the static ones, which led to slightly different results here than in the main experiments.

\section{Additional Analysis}\label{appendix:sec:analysis}
\begin{table*}[tbp]
    \centering
    \resizebox{0.85\textwidth}{!}{
        \begin{tabular}{l rrr rrr rrr rrr}
    \toprule
     {} & \multicolumn{3}{c}{$h_{\text{ff}}$} & \multicolumn{3}{c}{$h_{\text{cg}}$} & \multicolumn{3}{c}{$h_{\text{cea}}$} & \multicolumn{3}{c}{$h_{\text{add}}$}\\
     \cmidrule(lr){2-4} \cmidrule(lr){5-7} \cmidrule(lr){8-10} \cmidrule(lr){11-13}
     Domain (\# Solved Inst.) &  \multicolumn{1}{c}{Avg.}  &  \multicolumn{1}{c}{Max} &  \multicolumn{1}{c}{Min} &  \multicolumn{1}{c}{Avg.}  &  \multicolumn{1}{c}{Max} &  \multicolumn{1}{c}{Min} &  \multicolumn{1}{c}{Avg.}  &  \multicolumn{1}{c}{Max} &  \multicolumn{1}{c}{Min}  &  \multicolumn{1}{c}{Avg.}  &  \multicolumn{1}{c}{Max} &  \multicolumn{1}{c}{Min} \\
     \midrule
     \rowcolor{Gray}
     \name{barman} (84.4) & \textbf{0.60} & 0.87 & 0.21  
                          & 0.38 & 0.77 & 0.13  
                          & 0.02 & 0.05 & 0.00  
                          & 0.00 & 0.01 & 0.00\\
     \name{blocksworld} (92.9) & 0.02 & 0.10 & 0.00
                               & 0.00 & 0.02 & 0.00
                               & 0.38 & 0.50 & 0.00
                               & \textbf{0.60} & 1.00 & 0.47\\
     \rowcolor{Gray}
     \name{childsnack} (88.) & \textbf{0.95} & 1.00 & 0.62
                             & 0.02 & 0.20 & 0.00
                             & 0.01 & 0.04 & 0.00
                             & 0.02 & 0.16 & 0.00\\
     \name{rovers} (95.2) & \textbf{0.68} & 0.92 & 0.56
                          & 0.05 & 0.15 & 0.00
                          & 0.18 & 0.26 & 0.04
                          & 0.09 & 0.13 & 0.04\\
     \rowcolor{Gray}
     \name{sokoban} (87.7) & \textbf{0.98} & 1.00 & 0.80
                           & 0.00 & 0.20 & 0.00
                           & 0.02 & 0.03 & 0.00
                           & 0.00 & 0.02 & 0.00\\
     \name{visitall} (56.9) & 0.02 & 0.20 & 0.00
                            & 0.18 & 0.33 & 0.00
                            & \textbf{0.43} & 1.00 & 0.04
                            & 0.37 & 0.60 & 0.00\\
     \bottomrule
\end{tabular}
    }
    \caption{Overall heuristic selection by \name{rl} per expansion step, averaged over solved instances. We highlight in bold the most frequently selected heuristic for expansion per domain on average.}
    \label{tab:heuristic_usage}
\end{table*}
\begin{table*}[h!tb]
    \centering
    \resizebox{1\textwidth}{!}{
        \begin{tabular}{l rrrr rrrr rrrr rrrr}
    \toprule
     {} & \multicolumn{4}{c}{Q1} & \multicolumn{4}{c}{Q2} & \multicolumn{4}{c}{Q3} & \multicolumn{4}{c}{Q4}\\
     \cmidrule(lr){2-5} \cmidrule(lr){6-9} \cmidrule(lr){10-13} \cmidrule(lr){14-17}
     Domain (\# Solved Inst.) & 
     \multicolumn{1}{c}{$h_{\text{ff}}$}  &  \multicolumn{1}{c}{$h_{\text{cg}}$} &  \multicolumn{1}{c}{$h_{\text{cea}}$} &  \multicolumn{1}{c}{$h_{\text{add}}$} &  \multicolumn{1}{c}{$h_{\text{ff}}$}  &  \multicolumn{1}{c}{$h_{\text{cg}}$} &  \multicolumn{1}{c}{$h_{\text{cea}}$} &  \multicolumn{1}{c}{$h_{\text{add}}$} &  \multicolumn{1}{c}{$h_{\text{ff}}$}  &  \multicolumn{1}{c}{$h_{\text{cg}}$} &  \multicolumn{1}{c}{$h_{\text{cea}}$} &  \multicolumn{1}{c}{$h_{\text{add}}$} &  \multicolumn{1}{c}{$h_{\text{ff}}$}  &  \multicolumn{1}{c}{$h_{\text{cg}}$} &  \multicolumn{1}{c}{$h_{\text{cea}}$} &  \multicolumn{1}{c}{$h_{\text{add}}$}\\
     \midrule
     \rowcolor{Gray}
     \name{barman} (84.4) & \textbf{0.56} & 0.41 & 0.03 & 0.00 
                         & \textbf{0.61} & 0.37 & 0.02 & 0.00 
                         & \textbf{0.62} & 0.36 & 0.02 & 0.00 
                         & \textbf{0.62} & 0.36 & 0.02 & 0.00\\
     \rowcolor{Gray}
     \name{blocksworld} (92.9) & 0.07 & 0.01 & 0.36 & \textbf{0.56}
                              & 0.01 & 0.00 & 0.38 & \textbf{0.60}
                              & 0.01 & 0.00 & 0.39 & \textbf{0.61}
                              & 0.00 & 0.00 & 0.39 & \textbf{0.61}\\
     \name{childsnack} (88.) & \textbf{0.91} & 0.03 & 0.02 & 0.04
                             & \textbf{0.95} & 0.02 & 0.00 & 0.03
                             & \textbf{0.97} & 0.01 & 0.00 & 0.02
                             & \textbf{0.97} & 0.01 & 0.01 & 0.01\\
     \name{rovers} (95.2) & \textbf{0.71} & 0.03 & 0.16 & 0.10
                         & \textbf{0.68} & 0.05 & 0.18 & 0.09
                         & \textbf{0.67} & 0.05 & 0.19 & 0.09
                         & \textbf{0.68} & 0.05 & 0.18 & 0.09\\
     \rowcolor{Gray}
     \name{sokoban} (87.7) & \textbf{0.97} & 0.00 & 0.02 & 0.01
                          & \textbf{0.98} & 0.00 & 0.01 & 0.01
                          & \textbf{0.98} & 0.00 & 0.01 & 0.01
                          & \textbf{0.99} & 0.00 & 0.01 & 0.00\\
     \rowcolor{Gray}
     \name{visitall} (56.9) & 0.05 & 0.18 & \textbf{0.43} & 0.34
                           & 0.01 & 0.19 & \textbf{0.42} & 0.38
                           & 0.01 & 0.18 & \textbf{0.43} & 0.37
                           & 0.02 & 0.18 & \textbf{0.44} & 0.36\\
     \bottomrule
\end{tabular}
    }
    \caption{Average heuristic usage over all solved test problem instances split into quarters of the learned policies. Bold-faced entries are the most frequently used heuristic per quarter.}
    \label{appendix:tab:quarters}
\end{table*}
\begin{table*}[h!tb]
    \centering
    \resizebox{0.85\textwidth}{!}{
        \begin{tabular}{l rrr rrr rrr rrr}
    \toprule
     {} & \multicolumn{3}{c}{Immediate} & \multicolumn{3}{c}{High Freq.} & \multicolumn{3}{c}{Medium Freq.} & \multicolumn{3}{c}{Low Freq.}\\
     \cmidrule(lr){2-4} \cmidrule(lr){5-7} \cmidrule(lr){8-10} \cmidrule(lr){11-13}
     Domain (\# Solved Inst.) &  \multicolumn{1}{c}{Avg.}  &  \multicolumn{1}{c}{Max} &  \multicolumn{1}{c}{Min} &  \multicolumn{1}{c}{Avg.}  &  \multicolumn{1}{c}{Max} &  \multicolumn{1}{c}{Min} &  \multicolumn{1}{c}{Avg.}  &  \multicolumn{1}{c}{Max} &  \multicolumn{1}{c}{Min}  &  \multicolumn{1}{c}{Avg.}  &  \multicolumn{1}{c}{Max} &  \multicolumn{1}{c}{Min} \\
     \midrule
     \rowcolor{Gray}
     \name{barman} (84.4) & 0.28 & 0.51 & 0.11  
                          & \textbf{0.57} & 0.76 & 0.26  
                          & 0.10 & 0.23 & 0.09  
                          & 0.05 & 0.49 & 0.00\\
     \name{blocksworld} (92.9) & 0.14 & 0.56 & 0.00
                               & \textbf{0.40} & 0.95 & 0.04
                               & 0.28 & 0.67 & 0.00
                               & 0.18 & 0.85 & 0.00\\
     \rowcolor{Gray}
     \name{childsnack} (88.) & 0.08 & 0.31 & 0.00
                             & 0.33 & 0.95 & 0.00
                             & \textbf{0.41} & 0.99 & 0.00
                             & 0.18 & 0.99 & 0.00\\
     \name{rovers} (95.2) & 0.28 & 0.54 & 0.00
                          & \textbf{0.60} & 0.86 & 0.01
                          & 0.06 & 0.55 & 0.00
                          & 0.06 & 0.98 & 0.00\\
     \name{sokoban} (87.7) & 0.04 & 0.24 & 0.00
                           & \textbf{0.50} & 0.98 & 0.01
                           & 0.33 & 0.98 & 0.00
                           & 0.13 & 0.80 & 0.00\\
     \rowcolor{Gray}
     \name{visitall} (56.9) & 0.07 & 0.30 & 0.01
                            & 0.24 & 0.85 & 0.03
                            & 0.28 & 0.66 & 0.00
                            & \textbf{0.41} & 0.89 & 0.00\\
     \bottomrule
\end{tabular}
    }
    \caption{Switching frequency of the learned policies on the test dataset. Bold-faced entries give the, on average, most used switching frequency per domain.}
    \label{appendix:tab:switch_freq}
\end{table*}
Tables~\ref{appendix:tab:quarters} and \ref{tab:heuristic_usage} shows the average heuristic of a run.
Every section of the table refers to one quarter of a full policy trajectory in which a heuristic is selected at every step.
We can observe that the learned policies tend to favor one of the heuristics over all others.
For example in \name{barman}, on average the learned policies select the $h_{\text{ff}}$ heuristic more than $50\%$ of a successful run, whereas it selects $h_{\text{cg}}$ roughly for one third of a successful run.
Further, we can observe that in most domains the learned policies tend to focus on two out of the four available heuristics.
Only in \name{childsnack} and \name{sokoban} do we observe that a single heuristic is chosen throughout the entire run.

By splitting up the observed trajectories in quarters and analyzing these quarters individually, we can see that the learned heuristics tend to slightly increase the usage of the dominant heuristic over time.
In \name{blocksworld} for example, in the first quarter $h_{\text{add}}$ is selected roughly $56\%$ on average and increases to $60\%$ in the second quarter, before leveling out at $61\%$ in the third and fourth quarters.
Only on \name{rovers} we observe an inverse of this trend where $h_{\text{ff}}$ is selected $71\%$ in the first quarter before dropping to $68\%$ and $67\%$ in the third and fourth quarter respectively.
Sometimes this increase in usage of the more dominant heuristic is reflected with an appropriate reduction in usage of the second dominant heuristic, (see e.g. \name{barman}).
In other cases however, both dominant heuristics increase in usage while reducing the usage of the mostly unused heuristics (see e.g. \name{blocksworld}).

Table~\ref{appendix:tab:switch_freq} shows the average switching frequency observed on the test data of the individual IPC domains.
For this analysis, we recorded the length of heuristic usage before switching to another heuristic.
For ease of analysis, we differentiate between the four different frequency classes:
\begin{itemize}
    \item Immediate $\Rightarrow$ switched heuristics already after one step;
    \item High $\Rightarrow$ switched after $2$ to $100$ steps;
    \item Medium $\Rightarrow$ switched after $101$ to $1000$ steps;
    \item Low $\Rightarrow$ switched after more than $1000$ steps.
\end{itemize}
Here we report the average, maximal and minimal switching frequencies observed on the solved instances of the individual domains.
This allows us to gain insights into how often switching was necessary.
On \name{barman} we can for example observe that on average, $28\%$ of the performed steps in this domain were immediate switches between heuristics.
Further, we can also see that on every problem instance, some of the performed steps were immediate switches between heuristics (see column ``Min'' under ``Immediate'').
This is reflected accordingly in the maximum values of any of the frequency classes as none is ever higher than $60\%$.
For all other domains however, we can see that a successful did commit to longer trajectories for at least one test problem instance.

For some domains, such as \name{sokoban} or \name{childsnack} we can see that at least one instance in the domain was solved by a policy, by consistently playing only one of the possible heuristics.
This is consistent with the reported coverage and heuristic usage values in Tables \ref{tab:real_coverage} of the main paper, as well Table~\ref{appendix:tab:quarters} of this section.
Lastly, we can observe that the lowest frequency class is observed least frequently.
This indicates that the our agents did mostly commit to such long repeated actions if they learned that a problem instance is solved best by not switching.


\begin{thebibliography}{42}
\providecommand{\natexlab}[1]{#1}
\providecommand{\url}[1]{\texttt{#1}}
\providecommand{\urlprefix}{URL }
\expandafter\ifx\csname urlstyle\endcsname\relax
  \providecommand{\doi}[1]{doi:\discretionary{}{}{}#1}\else
  \providecommand{\doi}{doi:\discretionary{}{}{}\begingroup
  \urlstyle{rm}\Url}\fi

\bibitem[{Aine and Likhachev(2016)}]{aine-likhackev-icaps2016}
Aine, S.; and Likhachev, M. 2016.
\newblock Search Portfolio with Sharing.
\newblock In \emph{Proc.\ ICAPS 2016}, 11--19.

\bibitem[{Arfaee, Zilles, and Holte(2011)}]{arfaee-et-al-aij2011}
Arfaee, S.~J.; Zilles, S.; and Holte, R.~C. 2011.
\newblock Learning Heuristic Functions for Large State Spaces.
\newblock \emph{AIJ} 175: 2075--2098.

\bibitem[{B{\"a}ckstr{\"o}m and Nebel(1995)}]{backstrom-nebel-compint1995}
B{\"a}ckstr{\"o}m, C.; and Nebel, B. 1995.
\newblock Complexity Results for {SAS$^{+}$} Planning.
\newblock \emph{Computational Intelligence} 11(4): 625--655.

\bibitem[{Biedenkapp et~al.(2020)Biedenkapp, Bozkurt, Eimer, Hutter, and
  Lindauer}]{biedenkapp-et-al-ecai2020}
Biedenkapp, A.; Bozkurt, H.~F.; Eimer, T.; Hutter, F.; and Lindauer, M. 2020.
\newblock Dynamic Algorithm Configuration: Foundation of a New Meta-Algorithmic
  Framework.
\newblock In \emph{Proc.\ ECAI 2020}, 427--434.

\bibitem[{Bonet and Geffner(2001)}]{bonet-geffner-aij2001}
Bonet, B.; and Geffner, H. 2001.
\newblock Planning as Heuristic Search.
\newblock \emph{AIJ} 129(1): 5--33.

\bibitem[{Cenamor, {de la Rosa}, and
  Fern\'{a}ndez(2016)}]{cenamor-et-al-jair2016}
Cenamor, I.; {de la Rosa}, T.; and Fern\'{a}ndez, F. 2016.
\newblock The {IBaCoP} Planning System: Instance-Based Configured Portfolios.
\newblock \emph{JAIR} 56: 657--691.

\bibitem[{Cook and Huber(2016)}]{cook-huber-smc2016}
Cook, B.; and Huber, M. 2016.
\newblock Dynamic heuristic planner selection.
\newblock In \emph{Proc.\ SMC 2016}, 2329--2334.

\bibitem[{Domshlak, Karpas, and Markovitch(2010)}]{domshlak-et-al-aaai2010}
Domshlak, C.; Karpas, E.; and Markovitch, S. 2010.
\newblock To Max or Not to Max: Online Learning for Speeding Up Optimal
  Planning.
\newblock In \emph{Proc.\ {AAAI} 2010}, 1071–--1076.

\bibitem[{Fawcett et~al.(2011)Fawcett, Helmert, Hoos, Karpas, R{\"o}ger, and
  Seipp}]{fawcett-et-al-ipc2011a}
Fawcett, C.; Helmert, M.; Hoos, H.; Karpas, E.; R{\"o}ger, G.; and Seipp, J.
  2011.
\newblock {FD-Autotune}: Automated Configuration of {Fast} {Downward}.
\newblock In \emph{IPC 2011 planner abstracts}, 31--37.

\bibitem[{Fawcett et~al.(2014)Fawcett, Vallati, Hutter, Hoffmann, Hoos, and
  Leyton-Brown}]{fawcett-et-al-icaps2014}
Fawcett, C.; Vallati, M.; Hutter, F.; Hoffmann, J.; Hoos, H.; and Leyton-Brown,
  K. 2014.
\newblock Improved Features for Runtime Prediction of Domain-Independent
  Planners.
\newblock In \emph{Proc.\ ICAPS 2014}, 355--359.

\bibitem[{Ferber, Helmert, and Hoffmann(2020)}]{ferber-et-al-ecai2020}
Ferber, P.; Helmert, M.; and Hoffmann, J. 2020.
\newblock Neural Network Heuristics for Classical Planning: A Study of
  Hyperparameter Space.
\newblock In \emph{Proc.\ ECAI 2020}, 2346--2353.

\bibitem[{Gomoluch, Alrajeh, and Russo(2019)}]{gomoluch-etal-icaps2019}
Gomoluch, P.; Alrajeh, D.; and Russo, A. 2019.
\newblock Learning Classical Planning Strategies with Policy Gradient.
\newblock In \emph{Proc.\ ICAPS 2019}, 637--645.

\bibitem[{Gomoluch et~al.(2020)Gomoluch, Alrajeh, Russo, and
  Bucchiarone}]{gomoluch-etal-icaps2020}
Gomoluch, P.; Alrajeh, D.; Russo, A.; and Bucchiarone, A. 2020.
\newblock Learning Neural Search Policies for Classical Planning.
\newblock In \emph{Proc.\ ICAPS 2020}, 522--530.

\bibitem[{Helmert(2004)}]{helmert-icaps2004}
Helmert, M. 2004.
\newblock A Planning Heuristic Based on Causal Graph Analysis.
\newblock In \emph{Proc.\ ICAPS 2004}, 161--170.

\bibitem[{Helmert(2006)}]{helmert-jair2006}
Helmert, M. 2006.
\newblock The {Fast} {Downward} Planning System.
\newblock \emph{JAIR} 26: 191--246.

\bibitem[{Helmert and Geffner(2008)}]{helmert-geffner-icaps2008}
Helmert, M.; and Geffner, H. 2008.
\newblock Unifying the Causal Graph and Additive Heuristics.
\newblock In \emph{Proc.\ ICAPS 2008}, 140--147.

\bibitem[{Helmert and {R\"oger}(2008)}]{helmert-roeger-aaai2008}
Helmert, M.; and {R\"oger}, G. 2008.
\newblock How Good is Almost Perfect?
\newblock In \emph{Proc.\ {AAAI} 2008}, 944--949.

\bibitem[{Hinton and Roweis(2002)}]{hinton-roweis-nips2002}
Hinton, G.~E.; and Roweis, S.~T. 2002.
\newblock Stochastic Neighbor Embedding.
\newblock In \emph{Proc.\ NIPS 2002}, 833--840.

\bibitem[{Hoffmann and Nebel(2001)}]{hoffmann-nebel-jair2001}
Hoffmann, J.; and Nebel, B. 2001.
\newblock The {FF} Planning System: {Fast} Plan Generation Through Heuristic
  Search.
\newblock \emph{JAIR} 14: 253--302.

\bibitem[{Hutter et~al.(2009)Hutter, Hoos, Leyton-Brown, and
  St{\"u}tzle}]{hutter-et-al-jair2009}
Hutter, F.; Hoos, H.~H.; Leyton-Brown, K.; and St{\"u}tzle, T. 2009.
\newblock {ParamILS:} An Automatic Algorithm Configuration Framework.
\newblock \emph{JAIR} 36: 267--306.

\bibitem[{Hutter et~al.(2017)Hutter, Lindauer, Balint, Bayless, Hoos, and
  Leyton{-}Brown}]{hutter-et-al-aij2017}
Hutter, F.; Lindauer, M.; Balint, A.; Bayless, S.; Hoos, H.~H.; and
  Leyton{-}Brown, K. 2017.
\newblock The Configurable {SAT} Solver Challenge {(CSSC)}.
\newblock \emph{AIJ} 243: 1--25.

\bibitem[{Kautz and Selman(1992)}]{kautz-selman-ecai1992}
Kautz, H.; and Selman, B. 1992.
\newblock Planning as Satisfiability.
\newblock In \emph{Proc.\ ECAI 1992}, 359--363.

\bibitem[{Kingma and Ba(2014)}]{kingma-et-al-arxiv2015}
Kingma, D.~P.; and Ba, J. 2014.
\newblock Adam: A Method for Stochastic Optimization.
\newblock arXiv:1412.6980 [cs.LG].

\bibitem[{Ma et~al.(2020)Ma, Ferber, Huo, Chen, and Katz}]{ma-et-al-aaai2020}
Ma, T.; Ferber, P.; Huo, S.; Chen, J.; and Katz, M. 2020.
\newblock Online Planner Selection with Graph Neural Networks and Adaptive
  Scheduling.
\newblock In \emph{Proc.\ {AAAI} 2020}, 5077--5084.

\bibitem[{Pearl(1984)}]{pearl-1984}
Pearl, J. 1984.
\newblock \emph{Heuristics: {Intelligent} Search Strategies for Computer
  Problem Solving}.
\newblock Addison-Wesley.

\bibitem[{Rice(1976)}]{rice-aic1976}
Rice, J.~R. 1976.
\newblock The algorithm selection problem.
\newblock \emph{Advances in Computers} 15: 65--118.

\bibitem[{Richter and Helmert(2009)}]{richter-helmert-icaps2009}
Richter, S.; and Helmert, M. 2009.
\newblock Preferred Operators and Deferred Evaluation in Satisficing Planning.
\newblock In \emph{Proc.\ ICAPS 2009}, 273--280.

\bibitem[{Richter, Helmert, and Westphal(2008)}]{richter-et-al-aaai2008}
Richter, S.; Helmert, M.; and Westphal, M. 2008.
\newblock Landmarks Revisited.
\newblock In \emph{Proc.\ {AAAI} 2008}, 975--982.

\bibitem[{Richter, Westphal, and Helmert(2011)}]{richter-et-al-ipc2011}
Richter, S.; Westphal, M.; and Helmert, M. 2011.
\newblock {LAMA} 2008 and 2011 (planner abstract).
\newblock In \emph{IPC 2011 planner abstracts}, 50--54.

\bibitem[{Rintanen(2012)}]{rintanen-aij2012}
Rintanen, J. 2012.
\newblock Planning as Satisfiability: Heuristics.
\newblock \emph{AIJ} 193: 45--86.

\bibitem[{R{\"o}ger and Helmert(2010)}]{roeger-helmert-icaps2010}
R{\"o}ger, G.; and Helmert, M. 2010.
\newblock The More, the Merrier: Combining Heuristic Estimators for Satisficing
  Planning.
\newblock In \emph{Proc.\ ICAPS 2010}, 246--249.

\bibitem[{Seipp et~al.(2012)Seipp, Braun, Garimort, and
  Helmert}]{seipp-et-al-icaps2012}
Seipp, J.; Braun, M.; Garimort, J.; and Helmert, M. 2012.
\newblock Learning Portfolios of Automatically Tuned Planners.
\newblock In \emph{Proc.\ ICAPS 2012}, 368--372.

\bibitem[{Seipp et~al.(2015)Seipp, Sievers, Helmert, and
  Hutter}]{seipp-et-al-aaai2015}
Seipp, J.; Sievers, S.; Helmert, M.; and Hutter, F. 2015.
\newblock Automatic Configuration of Sequential Planning Portfolios.
\newblock In \emph{Proc.\ {AAAI} 2015}, 3364--3370.

\bibitem[{Sievers et~al.(2019)Sievers, Katz, Sohrabi, Samulowitz, and
  Ferber}]{sievers-et-al-aaai2019}
Sievers, S.; Katz, M.; Sohrabi, S.; Samulowitz, H.; and Ferber, P. 2019.
\newblock Deep Learning for Cost-Optimal Planning: Task-Dependent Planner
  Selection.
\newblock In \emph{Proc.\ {AAAI} 2019}, 7715--7723.

\bibitem[{Snoek, Larochelle, and Adams(2012)}]{snoek-et-al-nips2012}
Snoek, J.; Larochelle, H.; and Adams, R.~P. 2012.
\newblock Practical Bayesian Optimization of Machine Learning Algorithms.
\newblock In \emph{Proc.\ NIPS 2012}, 2960--2968.

\bibitem[{Speck et~al.(2020)Speck, Biedenkapp, Hutter, Mattm{\"u}ller, and
  Lindauer}]{speck-et-al-arxiv2020}
Speck, D.; Biedenkapp, A.; Hutter, F.; Mattm{\"u}ller, R.; and Lindauer, M.
  2020.
\newblock Learning Heuristic Selection with Dynamic Algorithm Configuration.
\newblock arXiv:2006.08246 [cs.AI].

\bibitem[{Speck, Gei{\ss}er, and Mattm{\"u}ller(2018)}]{speck-et-al-icaps2018}
Speck, D.; Gei{\ss}er, F.; and Mattm{\"u}ller, R. 2018.
\newblock Symbolic Planning with Edge-Valued Multi-Valued Decision Diagrams.
\newblock In \emph{Proc.\ ICAPS 2018}, 250--258.

\bibitem[{Thayer, Dionne, and Ruml(2011)}]{thayer-et-al-icaps2011}
Thayer, J.~T.; Dionne, A.~J.; and Ruml, W. 2011.
\newblock Learning Inadmissible Heuristics During Search.
\newblock In \emph{Proc.\ ICAPS 2011}, 250--257.

\bibitem[{Tokui et~al.(2019)Tokui, Okuta, Akiba, Niitani, Ogawa, Saito, Suzuki,
  Uenishi, Vogel, and Vincent}]{tokui-et-al-kdd2019}
Tokui, S.; Okuta, R.; Akiba, T.; Niitani, Y.; Ogawa, T.; Saito, S.; Suzuki, S.;
  Uenishi, K.; Vogel, B.; and Vincent, H.~Y. 2019.
\newblock Chainer: {A} Deep Learning Framework for Accelerating the Research
  Cycle.
\newblock In \emph{Proc.\ KDD 2019}, 2002--2011.

\bibitem[{Torralba et~al.(2017)Torralba, Alc\'{a}zar, Kissmann, and
  Edelkamp}]{torralba-et-al-aij2017}
Torralba, {\'A}.; Alc\'{a}zar, V.; Kissmann, P.; and Edelkamp, S. 2017.
\newblock Efficient Symbolic Search for Cost-optimal Planning.
\newblock \emph{AIJ} 242: 52--79.

\bibitem[{van Hasselt, Guez, and Silver(2016)}]{hasselt-et-al-aaai2016}
van Hasselt, H.; Guez, A.; and Silver, D. 2016.
\newblock Deep Reinforcement Learning with Double Q-Learning.
\newblock In \emph{Proc.\ {AAAI} 2016}, 2094--2100.

\bibitem[{Wolpert and Macready(1995)}]{wolpert-macready-tr1995}
Wolpert, D.~H.; and Macready, W.~G. 1995.
\newblock No free lunch theorems for search.
\newblock Technical Report SFI-TR-95-02-010, Santa Fe Institute.

\end{thebibliography}
\end{document}